\documentclass{article} 
\usepackage{iclr2027_conference,times}

\usepackage{hyperref}
\usepackage{url}
\usepackage{lipsum}		
\usepackage{graphicx}
\usepackage{doi}
\usepackage{booktabs}
\usepackage{enumitem}
\usepackage{float}
\usepackage{pifont}
\newcommand{\cmark}{\ding{51}}
\newcommand{\xmark}{\ding{55}}
\usepackage{epsfig}
\usepackage{amsmath}
\usepackage{amssymb}
\usepackage{amsthm}
\usepackage{xcolor}
\usepackage{algorithm}
\usepackage{algorithmic}
\usepackage{subcaption}

\newtheorem{lemma}{Lemma}
\newtheorem{proposition}{Proposition}
\newtheorem{definition}{Definition}

\newcommand{\STAMP}{\textsc{STAMP}}
\newcommand{\STAMPTS}{\textsc{STAMP-TS}}
\newcommand{\rhosp}{\rho_{\mathrm{s}}}

\usepackage{placeins}
\FloatBarrier

\usepackage{titlesec}
\titlespacing*{\subsection}{0pt}{6pt plus 1pt minus 1pt}{2pt plus 1pt}
\titlespacing*{\subsubsection}{0pt}{4pt plus 1pt minus 1pt}{1pt plus 1pt}
\titlespacing*{\paragraph}{0pt}{2pt plus 1pt minus 1pt}{1em}

\definecolor{Wgreen}{HTML}{01a049}

\newcommand{\warn}[1]{}

\title{STAMP: Predicting Out-of-Distribution Generalization without Target Data}

\author{%
  Md Kawsher Mahbub\\
  Department of Computer Science and Engineering\\
  NPI University of Bangladesh\\
  Manikganj, Bangladesh \\
  \texttt{kawsher@npiub.edu.bd}\\
  \And
  Milon Biswas \\
  Department of Computer \& Information Sciences\\
  Towson University \\
  Towson, MD 21252 \\
  \texttt{mbiswas1@students.towson.edu}\\
}
\iclrfinalcopy 

\begin{document}

\maketitle

\begin{abstract}
Predicting whether a trained model will generalize under distribution shift remains difficult, especially when target-domain data are unavailable. We introduce \STAMP{} (\textbf{S}emantic \textbf{T}emporal \textbf{A}ugmented \textbf{M}odel \textbf{P}rediction), a source-only, target-label-free criterion that estimates out-of-distribution~(OOD) performance from paired source-domain images. \STAMP{} computes the output-space correlation ratio $\eta^2=S_B/S_T$ by contrasting semantically stable pairs with random pairs: higher $\eta^2$ indicates that model outputs vary with semantic identity rather than nuisance variation. On 44 chest X-ray models spanning CNNs, ViTs, MetaFormers, foundation models, and SSL/VLM probes, temporal \STAMP{} attains Spearman correlations of $0.844$--$0.855$ with macro AUROC on VinDr-CXR, CheXpert, and MIMIC-CXR; a class-matched variant improves single-class RSNA from $0.311$ to $0.663$. \STAMP{} attains the best average source-only medical ranking and outperforms the target-domain ATC and AoTL estimators without any target data. On 27 ImageNet models, temperature-scaled \STAMPTS{} attains $\rho{=}0.984$ on ObjectNet and $\rho\geq0.905$ on four additional distribution shifts, with partial correlations of $0.662$--$0.949$ after controlling for ImageNet accuracy. Requiring approximately 12\,seconds per model on one GPU, \STAMP{} is a practical pre-deployment model-selection and auditing tool.
\end{abstract}

\section{Introduction}\label{sec:intro}

Models with strong in-distribution accuracy can fail under distribution shift by relying on shortcuts, scanner artifacts, acquisition protocols, or nuisance features rather than pathology \citep{zech2018variable,pooch2020can,nguyen2022vindr,geirhos2020shortcut,sagawa2020distributionally}. Analogous effects appear in natural-image recognition, where ImageNet-accurate models differ systematically on ObjectNet and common corruptions \citep{hendrycks2021many,barbu2019objectnet}. This motivates \emph{pre-deployment OOD prediction}: given a trained model and source-domain data only, rank candidate models by expected performance on unseen target distributions. Prior methods relax this in different ways confidence methods need unlabeled target data \citep{garg2022leveraging,guillory2021predicting}, agreement methods require joint evaluation of multiple models \citep{baek2022agreementontheline}, source-only proxies often require retraining or model weights \citep{jiang2019predicting,yu2022predicting} and Accuracy-on-the-Line needs labeled target data for calibration \citep{miller2021accuracy}. Thus, no existing method directly yields a target-free, model-agnostic ranking from source data alone.

We propose \STAMP{} (\textbf{S}emantic \textbf{T}emporal \textbf{A}ugmented \textbf{M}odel \textbf{P}rediction), based on the principle that \emph{a robust model should respond consistently to semantically equivalent inputs}. We quantify this with the correlation ratio
$$
\eta^2=\frac{S_B}{S_T},
$$
the fraction of output variance explained by semantic groups, estimated from semantically stable image pairs longitudinal radiographs from the same patient or images sharing an ImageNet class contrasted with random pairs. High $\eta^2$ means outputs are organized by semantic identity rather than nuisance variation. \STAMP{} uses only source-domain images, evaluates one model at a time, and needs no internals or target data. Across 9 benchmarks, 71 models, and 9 shift types spanning medical and natural images, \STAMP{} strongly predicts OOD performance. Spearman $\rhosp\geq0.844$ with macro-AUROC across 44 chest X-ray architectures on four external hospitals, and a temperature-scaled variant reaches $\rhosp=0.984$ on ObjectNet across 27 ImageNet-pretrained models. All predictions are source-only and take about 12s per model on one GPU.

Our contributions are:
\begin{enumerate}
    \item \textbf{A target-data-free OOD ranking criterion.} \STAMP{} ranks models from paired source images without target samples, internals, retraining, or target labels.
    \item \textbf{A theoretically grounded output statistic.} We prove consistency of the class-conditional correlation-ratio estimator and derive an OOD error bound under bounded covariate shift, while identifying where the longitudinal medical estimator needs empirical rather than purely i.i.d.\ assumptions.
    \item \textbf{Broad empirical validation.} We evaluate 71 models over nine medical and natural-image OOD benchmarks, reporting rank correlations, partial correlations, confidence intervals, permutation tests, leave-one-out stability, and failure analyses.
    \item \textbf{A reproducible pretraining-objective effect.} \STAMP{} completely separates CE-supervised and frozen SSL/VLM probes (Cliff's $\delta{=}1.0$, $n{=}43$), showing training objective strongly moderates source-only OOD predictability.
\end{enumerate}

\section{Related Work}\label{sec:related}

\paragraph{Target-data OOD prediction.}
DoC \citep{guillory2021predicting} and ATC \citep{garg2022leveraging} reduce MAE by $2$--$4\times$ over prior confidence methods; model agreement also linearly predicts ID/OOD accuracy \citep{baek2022agreementontheline}. These require target-domain data, whereas \STAMP{} is source-only.

\paragraph{Source-only and internal diagnostics.}
Accuracy-on-the-Line links ID and OOD accuracy and motivates our partial correlation analysis \citep{miller2021accuracy}. PAC-Bayes flatness needs weights and extra training \citep{jiang2019predicting}, Feature Discriminability uses rotation-prediction probes on target images \citep{deng2021labels} and Projection Norm retrains with pseudo-labeled target data \citep{yu2022predicting}. Inside-Out uses inter-layer dependency bias via circuit discovery, averaging SRCC $=0.766$ on PACS, Camelyon17, and Terra Incognita, but requires transformer graphs and ViTs \citep{peng2026inside}. \STAMP{} is architecture-agnostic and reaches SRCC $=0.855$ on VinDr, versus DDB\textsubscript{out} SRCC $=0.820$ on Camelyon17. TopoGeoScore \citep{hazratian2026topogeoscore} and checkpoint selection via manifold curvature/torsion \citep{anon2026geordiag} use internal representations and multi-signal geometric scorers. \STAMP{} instead uses output probability vectors from paired inputs with no learned scorer.

\paragraph{Shortcut learning, medical shift, and correlation ratio.}
Deep nets exploit non-semantic shortcuts \citep{geirhos2020shortcut}, ERM fails under spurious subpopulation shifts \citep{sagawa2020distributionally}, and chest X-ray classifiers degrade across hospitals and external sites \citep{zech2018variable,pooch2020can,rajpurkar2021chexternal}, motivating standardized multi-site radiograph benchmarks \citep{cohen2022torchxrayvision}. The correlation ratio $\eta^2=S_B/S_T$ measures variance explained by group membership \citep{fisher1936use}; \citet{deng2021labels} use a related idea in feature space, while \STAMP{} applies it to output probabilities for black-box, architecture-agnostic evaluation. Table~\ref{tab:proxy_comparison} (Appendix~\ref{app:baselines}) summarizes theoretical properties. \STAMP{} is the only label-free, target-data-free, architecture-agnostic proxy with a closed-form generalization bound, computable in $O(n)$ forward passes without gradients.

\section{Our Approach}\label{sec:method}

\subsection{Problem Setup}

Let $f_\theta : \mathcal{X} \to \Delta^{C-1}$ be a classifier mapping images to probability vectors over $C$ classes, trained on source distribution $P_s$. We seek a scalar $g(f_\theta)$, computable from $P_s$ alone, that is monotonically predictive of accuracy under a target distribution $P_t \neq P_s$.

We consider two instantiations. \textbf{Medical imaging:} $C{=}14$ NIH chest pathology labels; outputs are sigmoid probabilities; pairs are longitudinal radiograph sequences from the same patient. \textbf{Natural images:} $C{=}1{,}000$ ImageNet classes; outputs are temperature scaled softmax vectors; pairs are same class ImageNet validation images.

\subsection{The \STAMP{} Score}

\begin{definition}[Semantic and Random Pairs]
A \emph{semantic pair} $(\mathbf{x}_a, \mathbf{x}_b)$ consists of two
images sharing the same semantic identity (same patient or same class
label). A \emph{random pair} $(\mathbf{x}_c, \mathbf{x}_d)$ consists
of images drawn independently from $P_s$ (different patients or
different classes).
\end{definition}

\begin{definition}[\STAMP{}]
\label{def:stamp}
Given $N$ semantic pairs and $N$ random pairs drawn from $P_s$:
\begin{align}
\mathrm{SV} &= \tfrac{1}{N}\textstyle\sum_{i=1}^{N}
   \bigl\|f_\theta(\mathbf{x}_a^{(i)}) -
           f_\theta(\mathbf{x}_b^{(i)})\bigr\|_2^2, \\
\mathrm{AV} &= \tfrac{1}{N}\textstyle\sum_{j=1}^{N}
   \bigl\|f_\theta(\mathbf{x}_c^{(j)}) -
           f_\theta(\mathbf{x}_d^{(j)})\bigr\|_2^2, \\
\STAMP(f_\theta) &= 1 - \frac{\mathrm{SV}}{\mathrm{AV}+\varepsilon},
\end{align}
where $\varepsilon{=}10^{-8}$ prevents division by zero.
\end{definition}

\STAMP{} equals zero when semantic and random pairs induce identical output divergence, and approaches one when within-class output variance is negligible compared to total variance. Lemma~\ref{lem:fdr} establishes that \STAMP{} consistently estimates $\eta^2 = S_B/S_T$; the proof is given in Appendix~\ref{app:theory}.

\subsection{Temperature Scaling for Natural Image Models}

ImageNet models produce logits with architecturally variable scales, particularly in SSL models where logit magnitude is unregularized. We apply temperature scaling~\cite{guo2017calibration}: $\hat{p} = \mathrm{softmax}(z/T)$, where $T$ is learned by minimizing Negative Log-Likelihood (NLL) on 2{,}000 held-out ImageNet validation images disjoint from the pair pool, preventing data leakage. Models whose calibrated temperature $T > 3.0$ have non-functional classifier heads  (Algorithm \ref{algo:temparscale}) (IN-1K accuracy $\leq 0.002$, empirically verified); we exclude them and report them as an SSL ablation (Section~\ref{sec:ablation}). The retained and excluded models are separated by a wide calibration gap (retained maximum $T{=}1.37$ versus excluded minimum $T{=}5.76$), with no borderline cases.

\begin{algorithm}[!ht]
\caption{\STAMPTS{} Computation}
\begin{algorithmic}[1]\label{algo:temparscale}
\REQUIRE Model $f_\theta$, source val set $\mathcal{D}$,
  $N$ semantic pairs $\mathcal{P}_s$, $N$ random pairs $\mathcal{P}_r$
\STATE $T \leftarrow \mathrm{calibrate}(f_\theta,\, \mathcal{D}\setminus\text{pair pool})$
\IF{$T > 3.0$} \STATE \textbf{exclude} (non-functional head) \ENDIF
\STATE $\mathrm{SV} \leftarrow 0$;\; $\mathrm{AV} \leftarrow 0$
\FOR{$(\mathbf{x}_a, \mathbf{x}_b) \in \mathcal{P}_s$}
  \STATE $\mathrm{SV} \mathrel{+}= \|\mathrm{softmax}(f_\theta(\mathbf{x}_a)/T)
         - \mathrm{softmax}(f_\theta(\mathbf{x}_b)/T)\|_2^2 / N$
\ENDFOR
\FOR{$(\mathbf{x}_c, \mathbf{x}_d) \in \mathcal{P}_r$}
  \STATE $\mathrm{AV} \mathrel{+}= \|\mathrm{softmax}(f_\theta(\mathbf{x}_c)/T)
         - \mathrm{softmax}(f_\theta(\mathbf{x}_d)/T)\|_2^2 / N$
\ENDFOR
\RETURN $1 - \mathrm{SV} / (\mathrm{AV} + \varepsilon)$
\end{algorithmic}
\end{algorithm}

\subsection{Pair Construction}

\paragraph{Medical imaging.}
For all headline medical results, we use temporal \STAMP{} (\STAMP-T): consecutive radiographs from the same patient with follow-up gap $\leq 3$, regardless of whether the recorded NIH label changes. This design avoids assuming that NIH labels are reliable and tests stability under realistic within-patient acquisition variation. The temporal pool contains 13{,}302 eligible pairs, from which we sample $N{=}2{,}000$. A label-matched variant (\STAMP-L), restricted to pairs with identical Finding Labels, contains 9{,}927 eligible pairs and is used only in the ablation study. $N{=}2{,}000$ random cross-patient pairs complete the pair set.

\paragraph{Natural images.}
Semantic pairs: two images sampled from the same ImageNet class, drawn uniformly over all 1{,}000 classes. Random pairs: two images from different classes. $N{=}2{,}000$ each, seeded for reproducibility, cached across benchmarks.

\section{Theoretical Analysis}\label{sec:theory}

We formalize when \STAMP{} can predict OOD performance from source data alone. Throughout, $\rho$ denotes Spearman rank correlation, and $P_s(\mathbf{x}\mid y)$ and $P_s(y\mid\mathbf{x})$ denote the source class-conditional distribution and class-posterior, respectively. We assume: \textbf{A1} label-preserving covariate shift, $P_t(y\mid\mathbf{x})=P_s(y\mid\mathbf{x})$; \textbf{A2} bounded shift, $W_2(P_s,P_t)\le d_{\max}$; \textbf{A3} an $L$-Lipschitz output map, $\|J_f(\mathbf{x})\|_{\mathrm{op}}\le L$  a.e.; and \textbf{A4} OOD error increases with projected within-class variance $\sigma*{y,\mathrm{proj}}^2$ and decreases with projected inter-class margin $\gamma_y$. The empirical identifiability conditions \textbf{A5(i)--(iii)} are verified post-hoc in Tables~\ref{tab:ablation}, \ref{tab:identifiability}, and~\ref{tab:identifiability}; full definitions and proofs are given in Appendix~\ref{app:theory}.

\paragraph{Consistency.}
Let $S_W$, $S_B$, and $S_T$ denote within-class, between-class, and total output scatter, with $S_T=S_W+S_B$. Under i.i.d.\ stable and random pair sampling and $\mathrm{Var}(f_\theta(\mathbf{x}))>0$,

\begin{equation}
\STAMP_n
\;\xrightarrow{\mathrm{a.s.}}\;
\eta^2
= \frac{S_B}{S_T}
= \frac{
    \mathrm{Var}_y\!\left[
      \mathbb{E}_{\mathbf{x}\sim P_s(\cdot\mid y)}
        f_\theta(\mathbf{x})
    \right]
  }{
    \mathrm{Var}_{\mathbf{x}\sim P_s}
      [f_\theta(\mathbf{x})]
  },
\label{eq:fdr}
\end{equation}

and $\STAMP_n-\eta^2=O_p(n^{-1/2})$ (Lemma~\ref{lem:fdr}). Here $\eta^2\in[0,1]$ is the correlation ratio~\citep{fisher1936use}, not the Fisher criterion $S_B/S_W\in[0,\infty)$. Thus \STAMP{} jointly favors larger between-class variation and smaller within-class variation, while normalization prevents output scale from inflating either component. Empirically, $-\mathrm{SV}$, $\mathrm{AV}$, and full \STAMP{} yield $\rho=-0.19$, $+0.64$, and $+0.89$, respectively, on VinDr (Table~\ref{tab:ablation}).

\paragraph{OOD error bound.}
For class $y$, let $y^*$ denote the class with the nearest competing centroid, and define the corresponding unit direction $\hat{\mathbf v}_y=(\mu_y-\mu_{y^*})/\|\mu_y-\mu_{y^*}\|_2$. The projected output is $Z_y=\hat{\mathbf v}_y^\top f_\theta(\mathbf{x})$. Under A1--A4, the centroid alignment condition A5$'$ (defined in Appendix~\ref{app:setup}) and bounded within-class shift A6, 

\begin{equation}
P_t(\mathrm{error}\mid y)
\le
\frac{4(\sigma*{y,\mathrm{proj}}^2+B_y)}{\gamma_y^2}
+
\frac{2Ld_{\max}}{\gamma_y},
\qquad
B_y\le4Ld_{\max}.
\label{eq:margin}
\end{equation}
The result follows from a source Chebyshev bound, the exact Lipschitz control of target variance and centroid displacement, and requires no Taylor approximation (Lemma~\ref{lem:margin}). Since

$$
\frac{S_W}{S_B}
=
\frac{1-\STAMP}{\STAMP},
$$

class averaging gives, in the source-variance-dominated regime,
\begin{equation}
\bar e(f_\theta)
\approx
4\frac{1-\STAMP}{\STAMP}
+
\frac{2Ld_{\max}}{\bar\gamma},
\label{eq:avg_bound}
\end{equation}
with bounded relative approximation error $\le\mathrm{CV}^2$, where $\mathrm{CV}=1.07$ across 44 models (Lemma~\ref{lem:approx}). A Hoeffding tightening further gives
\begin{equation}
P_t(\mathrm{error}\mid y)
\le
\exp!\left(-\frac{\gamma_y^2}{8}\right)
+
\frac{2Ld_{\max}}{\gamma_y}.
\label{eq:hoeffding}
\end{equation}
These results connect the source-only output statistic to OOD error under bounded covariate shift. They are also consistent with the standard
$\mathcal H\Delta\mathcal H$ framework~\citep{ben2010theory}, for which A3 gives
$d_{\mathcal H\Delta\mathcal H}\le2LW_2(P_s,P_t)$~\citep{redko2017theoretical,shen2018wasserstein}.

\paragraph{Ranking guarantee.}
Because

$$
h(\STAMP)=\frac{1-\STAMP}{\STAMP}
$$

is strictly decreasing on $(0,1]$, higher \STAMP{} produces a strictly smaller variance term in~\eqref{eq:avg_bound}. Consequently, for models $f_1,f_2$ with $\STAMP(f_1)>\STAMP(f_2)$,
\begin{equation}
\mathbb{E}[\bar e_T(f_1)]
\le
\mathbb{E}[\bar e_T(f_2)]
+
\Delta_{\mathrm{shift}}(f_1,f_2),
\label{eq:endtoend}
\end{equation}
where

$$
\Delta_{\mathrm{shift}}
=
\frac{8Ld_{\max}}
{\min(\bar\gamma(f_1),\bar\gamma(f_2))}.
$$

Thus the ordering induced by the variance term is exact algebraically, while the full OOD ranking additionally requires A5(ii), namely that the inter-class margin is non-decreasing with \STAMP{} rank. This condition is independently measured from source outputs, with Spearman $\rho=+0.73$, $p<0.001$ across 40 models (Table~\ref{tab:identifiability}). The shift penalty vanishes as $d_{\max}\to0$ or $\bar\gamma\to\infty$ (Proposition~\ref{prop:endtoend}).

\paragraph{Rank consistency and failure modes.}
Under A1--A5, higher \STAMP{} therefore predicts higher target accuracy, with the remaining discrepancy arising from class heterogeneity, calibration, and the looseness of the concentration bound. Empirically, $\rho\ge0.844$ across three medical OOD datasets ($n=44$, $p<10^{-10}$) and $\rho\ge0.905$ across five natural-image benchmarks ($n=27$, permutation $p\le0.0001$). The theory also predicts three failure modes: \textbf{(i) class mismatch}, where macro \STAMP{} on RSNA gives $\rho=0.311$ but class-matched Pneumonia \STAMP{} recovers $\rho=0.663$ ($p<10^{-4}$); \textbf{(ii) concept drift}, which violates A1; and \textbf{(iii) degenerate outputs}, where $S_T\approx0$.

The framework further predicts that \STAMP{} becomes more informative as source accuracy ceases to explain OOD variation. The partial correlation
$\rho(\STAMP,\mathrm{acc}_t\mid\mathrm{acc}_s)$ is therefore expected to increase with shift magnitude. We state this as a structural heuristic rather than a formal monotonicity theorem; empirically it follows the sequence $0.935\rightarrow0.804\rightarrow0.755\rightarrow0.700\rightarrow0.662$ from ObjectNet to Dollar Street, ImageNet-A, ImageNet-R, and ImageNet-C (Table~\ref{tab:natural}).

\paragraph{Finite-sample reliability.}
Because model outputs lie on the probability simplex, paired squared distances are bounded, and Hoeffding concentration combined with denominator control gives $O_p(n^{-1/2})$ estimation error. Rather than reporting an extremely conservative dimension-free bound, we evaluate stability directly: across five independent pair samples, the standard deviation of \STAMP{} is below $0.004$, and pair-count sensitivity analyses show the model ranking is unchanged for $N\in\{500,1000,2000,5000\}$ (Appendix~\ref{app:robustness}).

\section{Experimental Setup}\label{sec:experiments}

\subsection{Medical Imaging}

\paragraph{Source data.}
NIH ChestX-ray14~\citep{wang2017chestxray}: 112{,}120 frontal radiographs, 30{,}805 patients, 14 binary disease labels.

\paragraph{Models.}

44 architectures across 10 families (Appendix Table~\ref{tab:stamp_main}): ResNet-\{18,34,50,101,152\}~\cite{he2016deep}, DenseNet-\{121,169,201\}~\cite{huang2017densely}, EfficientNet-\{B0,B4,V2-S\}~\cite{tan2019efficientnet}, ConvNeXt-\{T,S,B\}~\cite{liu2022convnet}, attention ResNets (CBAM, SE, BAM), ViT-\{Ti,S,B\}/16~\cite{dosovitskiy2021image}, DeiT-B/16~\cite{touvron2021training}, Swin-\{T,S,B\}~\cite{liu2021swin}, CoAtNet-0~\cite{dai2021coatnet}, MaxViT-T~\cite{tu2022maxvit}, MViTv2-T~\cite{li2022mvitv2}, MetaFormers (ConvFormer-S18, CaFormer-S18)~\cite{yu2023metaformer}, MLP-Mixer-B/16~\cite{tolstikhin2021mlp}, RAD-DINO~\cite{perez2025exploring}, 4 TorchXRayVision DenseNets~\cite{cohen2022torchxrayvision}, and 8~SSL/VLM probes (DINOv2-B/14, MAE-B/16, DINO-B/16, SAM-B/16, CLIP-B/16, BiomedCLIP-B/16, MoCov3-B/16, PubMedCLIP-B/32). 31~custom models use a two-stage protocol (frozen-backbone then full fine-tuning) with AdamW, cosine annealing, differential learning rates, label smoothing, and binary cross-entropy. Unless stated otherwise, all correlations in Section~\ref{sec:results} are computed across the full population of $n{=}44$ models.

\paragraph{OOD benchmarks.}

\textbf{VinDr-CXR}~\cite{nguyen2022vindr}: 18{,}000 Vietnamese chest radiographs; unanimous-agreement labels (5{,}685 evaluable images, 9~mapped classes). \textbf{CheXpert}~\cite{irvin2019chexpert}: Stanford radiographs, 7~mapped classes, uncertain labels set to 0. \textbf{RSNA Pneumonia}~\cite{shih2019rsna}: single-class detection. \textbf{MIMIC-CXR}~\cite{johnson2019mimic}: 2{,}761 test images, 6~shared NIH classes.

\subsection{Natural Image Benchmarks}

\paragraph{Models.}

27 ImageNet models (after SSL exclusion) from \texttt{timm}: ResNets, DenseNet-121, EfficientNets, ConvNeXts, ViTs, Swin Transformers, DeiTs, MLP-Mixer, ConvFormer, CaFormer, RegNetY, CoAtNet-0, NFNet-L0, ViT-L.

\paragraph{OOD benchmarks.}

\textbf{ObjectNet}~\cite{barbu2019objectnet}: 113-class viewpoint/background shift. \textbf{Dollar Street}~\cite{rojas2022dollar}: geographic/income shift, 1{,}600 images. \textbf{ImageNet-A}~\cite{hendrycks2021natural}: 7{,}500 naturally adversarial examples, 200~classes. \textbf{ImageNet-R}~\cite{hendrycks2021many}: 30{,}000 artistic renditions, 200~classes. \textbf{ImageNet-C}~\cite{hendrycks2019robustness}: 15 corruption types $\times$ 5 severities = 75~slices.

\paragraph{Statistical protocol.}

All Spearman correlations use two-sided permutation tests ($n_{\mathrm{perm}}{=}10{,}000$). Confidence intervals use Fisher $z$-transform and paired bootstrap ($n{=}10{,}000$). Partial Spearman uses the Kendall~(1948) formula. Leave-one-out stability reports range over all 27~drop one experiments. At $n{=}27$ and $n{=}44$, statistical power exceeds
$0.966$ and $0.997$ respectively for any true $\rhosp \ge 0.7$ (Table~\ref{tab:statistical},
Appendix~\ref{app:power}), ruling out under-powered false positives for all reported correlations.

\subsection{Baselines}

\textit{Source-only (21 metrics):} Entropy, Mean Confidence, ECE, Prediction Std, MC~Dropout, TTA~Flip, TTA~Multi, Embedding Instability, GradCAM Dice, FD~\cite{deng2021labels}, Feature Alignment Score, Aug-\STAMP{}-\{flip,rotation,brightness,contrast,crop\}, Synthetic~\STAMP{}. \textit{Target-domain:} ATC~\cite{garg2022leveraging}, AoTL~\cite{baek2022agreementontheline}. \textit{Labeled reference:} NIH source AUROC, Accuracy-on-the-Line~\cite{miller2021accuracy}.

\section{Results}\label{sec:results}

\subsection{Medical Imaging: 44-Model Study}\label{sec:nih_results}

Table~\ref{tab:main_results} and Figure\ref{fig:mainscatter} reports Spearman $\rhosp$ between \STAMP{} and OOD macro-AUROC across the full $n{=}44$ population. \STAMP{} achieves $\rhosp \geq 0.844$ on all datasets with adequate class overlap (VinDr, CheXpert, MIMIC), with all permutation $p{<}10^{-10}$; RSNA is the one benchmark with a known class-overlap mismatch (Section~\ref{sec:failure}, failure analysis).

\begin{table}[!ht]
\centering
\caption{Spearman $\rhosp$ between \STAMP{} and macro-AUROC ($n{=}44$).
RSNA~(CM) uses class-matched Pneumonia-only \STAMP{}.}
\label{tab:main_results}
\small
\setlength{\tabcolsep}{5pt}
\begin{tabular}{lcccc}
\toprule
\textbf{Dataset} & \textbf{$\rhosp$} & \textbf{95\,\% CI} & $p$ & $n$ \\
\midrule
VinDr-CXR        & 0.855 & [0.72, 0.92] & $<10^{-10}$ & 44 \\
CheXpert         & 0.851 & [0.71, 0.92] & $<10^{-10}$ & 44 \\
MIMIC-CXR        & 0.844 & [0.70, 0.91] & $<10^{-10}$ & 44 \\
RSNA (macro)     & 0.311 & [0.02, 0.56] & 0.044       & 44 \\
RSNA (CM, Pneu.) & 0.663 & [0.44, 0.80] & $<10^{-4}$  & 44 \\
\bottomrule
\end{tabular}
\end{table}

\begin{figure}[!ht]
\centering
\includegraphics[width=0.85\textwidth]{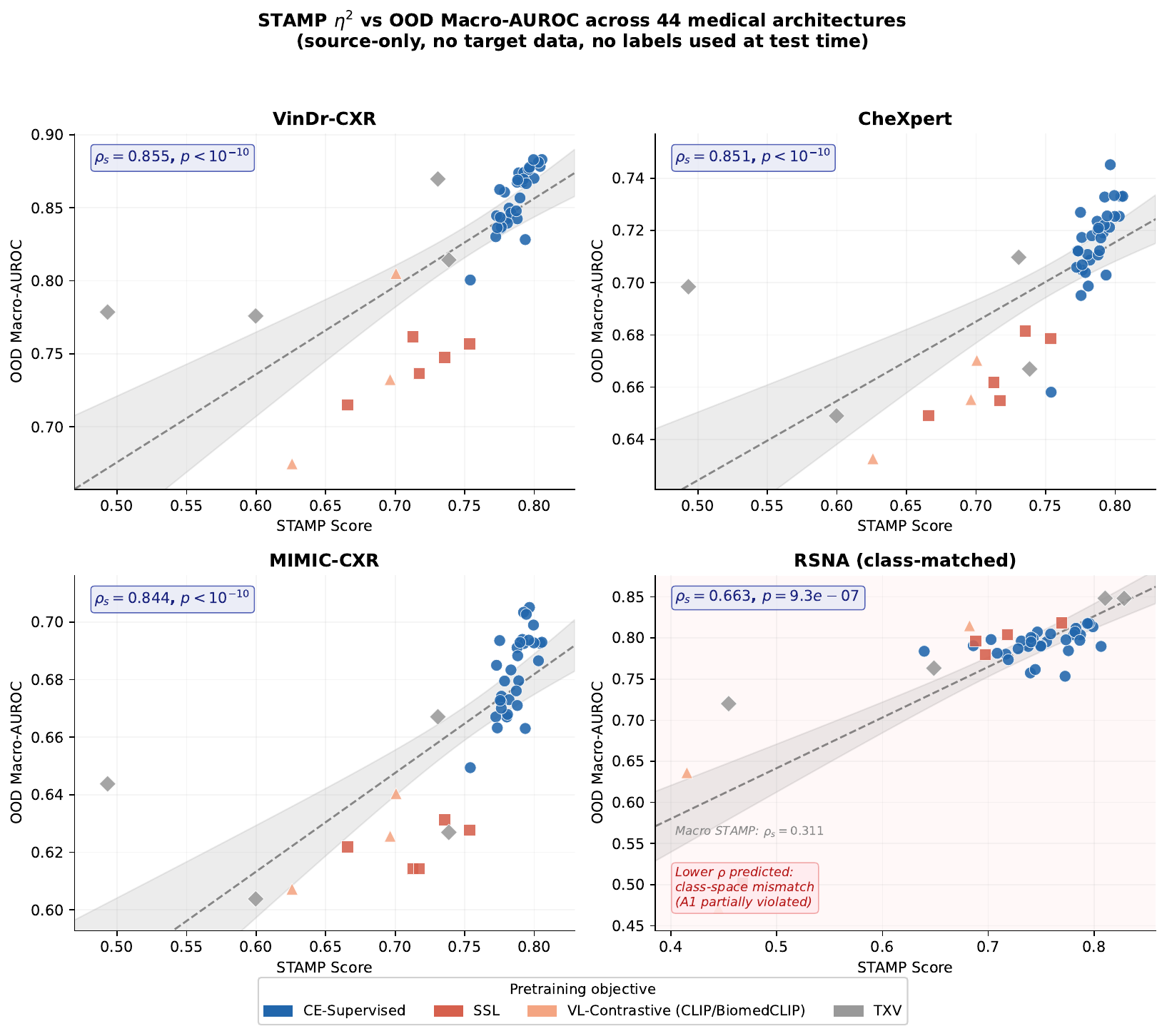}
\caption{\textbf{STAMP vs.\ OOD macro-AUROC across 44 medical architectures}, colored by pretraining objective. VinDr, CheXpert, and MIMIC show strong, consistent positive trends; RSNA (class-matched Pneumonia-only STAMP shown) is weaker and noisier, consistent with the class-overlap mismatch discussed in Section~\ref{sec:failure} (Proposition~\ref{prop:rank}).}
\label{fig:mainscatter}
\end{figure}

\paragraph{\STAMP{} scores.}
CE-supervised fine-tuned models score $\mathrm{STAMP} \in [0.754, 0.805]$, well above SSL/VLM probes ($[0.626, 0.753]$), confirming that \STAMP{} isolates a supervision-specific signal.

\paragraph{Probit-scale accuracy-on-the-line.}
\STAMP{} tracks the average NIH$\to$OOD accuracy line nearly as closely as NIH AUROC itself: probit-scale correlations (Figure~\ref{fig:probit}) are $r{=}0.752$ (VinDr), $0.708$ (CheXpert), and $0.760$ (MIMIC) versus $0.796$, $0.785$, and $0.832$ for NIH AUROC, with Steiger tests finding no significant difference ($p \geq 0.541$, $n{=}44$). The sole exception is RSNA, where \STAMP{} substantially outperforms the ID baseline ($r{=}0.835$ vs.\ $0.450$; Steiger test not applicable to this single-class comparison), consistent with class mismatch penalising ID accuracy.

\paragraph{Residual correlation.}
\STAMP{} also predicts \emph{which} models deviate from the population trend: correlations with per-model residuals from the average accuracy line are $\rhosp = 0.411$ (VinDr), $0.511$ (CheXpert), and $0.455$ (MIMIC), all $p < 0.01$ ($n{=}44$). It therefore carries model-specific information beyond the population-level correlation.

\paragraph{Baseline comparison.}
Table~\ref{tab:baseline_comparison} shows that \STAMP{} attains the best average source-only ranking ($\rhosp{=}0.715$), ahead of Aug-\STAMP-rotation ($0.709$), Synthetic \STAMP{} ($0.670$), and all remaining source-only metrics, and clearly outperforms the target-domain estimators ATC ($0.380$) and AoTL ($0.320$) despite using no target data. Augmentation variants marginally exceed \STAMP{} on VinDr ($0.860$, $0.856$ vs.\ $0.855$), consistent with related but distinct invariances; notably, Aug-\STAMP-rotation scores $+0.430$ on RSNA versus \STAMP{}'s $+0.311$, indicating rotation- rather than temporal-stability sensitivity. Full results across all 21 baselines appear in the supplementary material. \STAMP{} runs in ${\approx}12$\,s per model on a single T4 GPU --- $8\times$ faster than ATC and $282\times$ faster than AoTL with runtime independent of target-set size (Table~\ref{tab:cost}, Appendix~\ref{app:cost}).

\begin{table}[!ht]
\centering
\caption{Spearman $\rhosp$ baseline comparison. $n$ varies by metric
depending on which caches were available (see column $n$); source-only
best in \textbf{bold}.}
\label{tab:baseline_comparison}
\small
\setlength{\tabcolsep}{2.2pt}
\begin{tabular}{lccccccc}
\toprule
\textbf{Method} & \textbf{Dom.} & \textbf{VinDr} & \textbf{Chex.} & \textbf{RSNA} & \textbf{MIMIC} & \textbf{Avg.} & $n$\\
\midrule
NIH AUROC (ref.)       & lbl  & .861 & .869 & .291 & .887 & .727 & 44\\
\midrule
Entropy                & src  & $-.242$ & $-.126$ & $+.030$ & $-.186$ & $-.131$ & 40\\
Mean Confidence        & src  & $-.444$ & $-.346$ & $+.044$ & $-.440$ & $-.296$ & 40\\
ECE                    & src  & $+.398$ & $+.212$ & $+.111$ & $+.227$ & $+.237$ & 44\\
MC Dropout             & src  & $+.359$ & $+.287$ & $+.021$ & $+.283$ & $+.238$ & 31\\
TTA Multi              & src  & $+.297$ & $+.344$ & $+.442$ & $+.294$ & $+.344$ & 40\\
Emb.\ Instability      & src  & $-.322$ & $-.284$ & $+.312$ & $-.475$ & $-.192$ & 44\\
FD~\cite{deng2021labels}& src & $-.148$ & $-.108$ & $+.230$ & $-.130$ & $-.039$ & 40\\
FAS                    & src  & $+.215$ & $+.248$ & $-.013$ & $+.295$ & $+.186$ & 40\\
Aug-\STAMP{}-rot.      & src  & $+.860$ & $+.758$ & $+.430$ & $+.790$ & $+.709$ & 31\\
Synth.\ \STAMP{}       & src  & $+.856$ & $+.758$ & $+.283$ & $+.783$ & $+.670$ & 31\\
\midrule
ATC~\cite{garg2022leveraging}  & tgt & .448 & .442 & .229 & .402 & .380 & 44\\
AoTL~\cite{baek2022agreementontheline} & tgt & .234 & .413 & .274 & .357 & .320 & 44\\
\midrule
\textbf{\STAMP{} (ours)} & \textbf{src} & \textbf{.855} & \textbf{.851} & .311 & \textbf{.844} & \textbf{.715} & 44\\
\bottomrule
\end{tabular}
\end{table}

\subsection{Natural Images: \STAMPTS{} Study} \label{sec:natural}

\STAMPTS{} achieves $\rhosp \geq 0.905$ on all five benchmarks
(Figure ~\ref{fig:natural}, Table~\ref{tab:natural}), with leave-one-out stability confirming no
single model drives the result (LOO ranges in table). The cross-family analysis on one representative model per architecture ($n{=}10$) achieves $\rhosp{=}0.986$ ($p{=}0.0001$), confirming the result is not within-family scaling.

\begin{figure}[!ht]
\centering
\includegraphics[width=\textwidth]{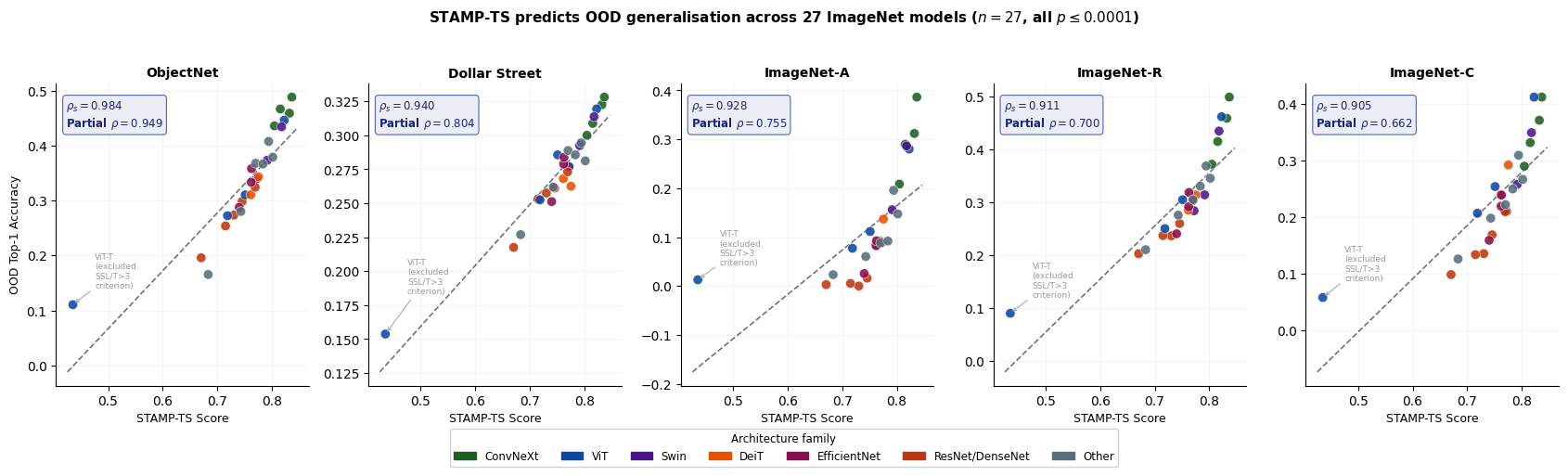}
\caption{\textbf{STAMP-TS predicts OOD top-1 accuracy across 27 ImageNet-pretrained models} on five distribution-shift benchmarks. Partial $\rho$ controls for ImageNet-1K top-1 accuracy and decreases monotonically with shift severity (Proposition~\ref{prop:rank}), from ObjectNet (0.949) to ImageNet-C (0.662). ViT-Tiny, excluded by the $T>3.0$ calibration criterion, is shown for reference only.}
\label{fig:natural}
\end{figure}

\begin{table}[!ht]
\centering
\caption{\STAMPTS{} Spearman $\rhosp$ across OOD benchmarks ($n{=}27$). All $p \leq 0.0001$ (two-sided permutation). Partial $\rhosp$ controls for ImageNet-1K top-1 accuracy.}
\label{tab:natural}
\small
\setlength{\tabcolsep}{2pt}
\begin{tabular}{lcccccc}
\toprule
\textbf{Benchmark} & \textbf{Shift}  &$\rhosp$ & \textbf{Part.}$\rhosp$ & \textbf{LOO range} & \textbf{95\,\% CI} & \textbf{Boot. CI} \\
\midrule
ObjectNet     & Hard covariate & 0.984 & 0.949 & [0.982, 0.988] & [0.964, 0.993] & [0.935, 0.996] \\
Dollar Street & Geographic & 0.940 & 0.804 & [0.932, 0.966] & [0.870, 0.972] & [0.799, 0.991] \\
ImageNet-A     &Natural adversarial & 0.928 & 0.755 & [0.919, 0.942] & [0.846, 0.967] & [0.808, 0.971] \\
ImageNet-R    & Style/texture & 0.911 & 0.700 & [0.900, 0.936] & [0.812, 0.959] & [0.736, 0.982] \\
ImageNet-C    & Synthetic corruption & 0.905 & 0.662 & [0.893, 0.926] & [0.800, 0.956] & [0.745, 0.976] \\
\bottomrule
\end{tabular}
\end{table}

\paragraph{Dollar Street income stratification.}
\STAMPTS{} predicts Dollar Street top-1 accuracy more strongly in the high income stratum ($\rhosp{=}0.873$) than in the low-income stratum ($\rhosp{=}0.777$, Figure~\ref{fig:dollarstreet}). Higher-\STAMP{} models achieve larger absolute accuracy gains on high-income households ($\rhosp{=}+0.752$ between STAMP-TS and the absolute income gap), while the relative gap is stable across income groups ($\rhosp{=}-0.378$, not significant), suggesting current architectures lift all groups but do not differentially close income-based disparities.

\paragraph{Per-corruption breakdown.}

\STAMPTS{} is positively predictive across all 15 corruption types (mean $\rhosp{=}0.876$, min $0.756$ on contrast, max $0.929$ on brightness), confirming the ImageNet-C result is not driven by any single corruption type. \STAMPTS{} is positively predictive across all 15 corruption types (mean $\rhosp{=}0.876$, min $0.756$ on contrast, max $0.929$ on brightness. Table~\ref{tab:corruption}, Appendix~\ref{app:inc}), confirming the ImageNet-C result is not driven by any single corruption category.

\section{Ablation Studies}\label{sec:ablation}

The ratio $1{-}\mathrm{SV}/\mathrm{AV}$ is necessary, not incidental, SV alone is negatively correlated with OOD accuracy ($\rhosp{=}{-}0.19$) and AV alone is only partially predictive ($\rhosp{=}0.64$), while the joint ratio reaches $\rhosp{=}0.89$ (Appendix~\ref{app:component}). Restricting to temporal, same-patient pairs further improves correlation by $+0.110$ to $+0.186$ over labeled pairs, consistent with NIH label noise contaminating the labeled semantic pool (Appendix~\ref{app:temporal}). Across the 43 medical models with a standard objective grouping, \STAMP{} separates cleanly by pretraining objective (Table~\ref{tab:objective}). CE-supervised models attain mean STAMP $0.786\pm0.011$, above SSL probes ($0.717$), VL-Contrastive probes ($0.674$), and TXV models ($0.644$). CE and SSL groups have no rank overlap (MWU $p{=}2.7{\times}10^{-6}$, Cliff's $\delta{=}+1.0$) and bootstrap 95\% CI on the CE$-$SSL difference $[+0.045,+0.097]$). The hybrid RAD-DINO checkpoint (STAMP $=0.775$) is reported separately because its frozen self-supervised backbone is paired with an NIH-supervised head. This holds within CE-supervised models alone ($\rhosp{=}+0.820$, $n{=}31$) and under a controlled same-backbone ViT-B/16 comparison ($\rhosp{=}1.000$ on VinDr. Appendix~\ref{app:controlled}), isolating objective from architecture (full regime breakdown in Appendix~\ref{app:regime}).

\begin{table}[!ht]
\centering
\caption{STAMP stratified by pretraining objective for the 43 models with a standard objective grouping. The hybrid RAD-DINO checkpoint, which combines a self-supervised radiology backbone with an NIH-supervised linear head, is analyzed separately. Kruskal--Wallis $H{=}25.5$, $p{=}1.2{\times}10^{-5}$. MWU CE-Sup $>$ SSL: $p{=}2.7{\times}10^{-6}$, Cliff's~$\delta{=}+1.0$.}\label{tab:objective}

\small
\begin{tabular}{lcccc}
\toprule
\textbf{Objective} & $n$ & $\mu$ & $\sigma$ & \textbf{Range} \\
\midrule
CE-Supervised  & 31 & 0.786 & 0.011 & [0.754, 0.805] \\
SSL            &  5 & 0.717 & 0.029 & [0.666, 0.754] \\
VL-Contrastive &  3 & 0.674 & 0.034 & [0.626, 0.701] \\
TXV            &  4 & 0.641 & 0.101 & [0.493, 0.739] \\
\bottomrule
\end{tabular}
\end{table}

At the feature level, ViT- and Swin-family models build STAMP increasingly from early to penultimate layers ($\Delta_{\mathrm{deep}-\mathrm{early}}\in[+0.12,+0.22]$), and a ridge regression against VinDr AUROC identifies the mid-to-penultimate STAMP gain and final-layer STAMP as the two dominant predictors ($\rhosp{=}+0.701$ and $+0.795$, $n{=}35$) — the \emph{Universal Generalization Motif}: robust models build semantic discriminability late rather than carrying it from the input (full layerwise and nuisance-suppression analysis in Appendix~\ref{app:layerwise}). Partial correlations controlling for NIH AUROC and parameter count confirm this signal is not a capacity confound ($\rhosp{=}0.536$, $n{=}43$; Appendix~\ref{app:partial}).

Framed as a pre-deployment alarm task, STAMP achieves Precision$\,{=}\,1.0$ (zero false alarms) at clinically strict thresholds ($\delta\leq0.821$), a $+0.21$ to $+0.33$ F1 advantage over always-alarm, enabling zero-risk model rejection before target data arrive (full threshold analysis in Appendix~\ref{app:screening}).

\section{Failure Analysis}\label{sec:failure}

The largest STAMP–VinDr rank disagreements ($|\text{rank}_{\STAMP}-\text{rank}_{\mathrm{VinDr}}|\geq15$) reveal two failure modes: ResNet50-SE ($\Delta{=}21$) produces artificially low stable variance via channel-attention suppression of spatial variation, specific to the NIH longitudinal distribution and non-transferring to VinDr; DeiT-Base/16 and MViTv2-T ($\Delta{=}20$ each) show highly confident outputs (max sigmoid $>0.97$), compressing STAMP's dynamic range. Both are detectable without target-domain data (full case detail in Appendix~\ref{app:failure}).

\section{Discussion}\label{sec:discussion}

\paragraph{Why does \STAMP{} work?}
Models relying on spurious correlations respond inconsistently to images sharing semantic content but differing in nuisance dimensions precisely the variation that same-patient pairs introduce. Lemma~\ref{lem:fdr} formalizes this: \STAMP{} estimates the fraction of total output variance attributable to between-class signal, i.e., class-discriminability in the output space. Lemma~\ref{lem:margin} connects this discriminability to bounded OOD error, and the Universal Generalization Motif shows it is built in the middle-to-penultimate layers rather than inherited from the input.

\paragraph{Pretraining objective is the primary moderator.}
The complete separation between CE-supervised and SSL models (Cliff's~$\delta{=}1.0$), confirmed under architecture control (Figure~\ref{fig:controlledvit}), is our strongest empirical finding. It accords with the view that SSL objectives optimize representation quality rather than decision-boundary sharpness, and with concurrent source-only geometric diagnostics reporting similar effects~\cite{hazratian2026topogeoscore,anon2026geordiag}.

\paragraph{Comparison with target-domain methods.}
ATC and AoTL gain a small edge on MIMIC-CXR because they observe the target distribution directly. The practical value of \STAMP{} is temporal: it is computable at training time, whereas target data are unavailable before deployment.

\paragraph{Computational advantage.}
\STAMP{} costs $\approx\!12\,$s per model using $4{,}000$ source images cached once across models. ATC requires $32{,}010$ target-domain images ($8\times$ more images and time); AoTL requires all $M{=}44$ models on the target set ($\sim\!1.1\,$M images, $56$~min). Crucially, \STAMP{}'s cost is \emph{independent} of target dataset size.

\paragraph{Limitations.}
\STAMP{} requires semantic pairs in the source domain; without repeated observations, augmentation-based synthetic pairs yield lower $\rhosp$ (Synth.-\STAMP{}~$0.670$ vs.~$0.715$). Class mismatch degrades macro-\STAMP{}, partially mitigated by class-matched variants (RSNA: $0.311 \to 0.663$). Under concept drift (A1 violated), \STAMP{} degrades predictably. The SE channel-attention failure mode suggests that models with strong global-pooling operations may require targeted pair construction. Finally, a negative-control ordering check (Appendix~\ref{app:layerwise}) shows the stable/random-variance components are noisier than their aggregate ratio; \STAMP{} should therefore be read as a ranking signal rather than a mechanistic decomposition of individual failure modes.

\section{Conclusion}\label{sec:conclusion}

We presented \STAMP{}, a label-free framework that predicts OOD generalization from source-domain semantic image pairs via $\eta^2{=}S_B/S_T$. It achieves $\rhosp \geq 0.844$ on three medical OOD datasets without class mismatch and $\rhosp \geq 0.905$ on five natural-image benchmarks (44 and 27 models, respectively), with no target-domain data. Partial correlations confirm substantial independent predictive signal ($\rhosp{=}0.536$ controlling NIH AUROC and parameter count). The pretraining objective is the dominant moderator: CE-supervised models consistently outperform SSL/VLM probes (Cliff's~$\delta{=}1.0$), an effect robust to architecture control and concentrated in a mid-to-penultimate-layer transition. At $\approx\!12\,$s per model, \STAMP{} enables large-scale pre-deployment auditing and serves as a standard diagnostic for model deployment decisions in clinical and consumer AI.

\subsection*{AI use statement}

In this work, we used generative AI tools for language editing, formatting, and
consistency checking during manuscript preparation. We have not used generative AI tools for generate experimental results, statistical analyses, figures, or scientific claims. Additionally, we used generative AI tools for Coding Assistant.  All AI-assisted text was reviewed, corrected, and verified by the authors, who take full responsibility for the final manuscript.

\subsection*{Ethics statement}

This study uses publicly available, de-identified benchmark datasets and involves no new human-subjects data collection. Our results are intended for model auditing and research on distribution shift, not as stand-alone clinical diagnostic decisions. The income-stratified Dollar Street analysis is included to characterize disparities rather than to validate deployment of the evaluated models.

\subsection*{Reproducibility statement}

The method is specified in Definition~\ref{def:stamp}, Algorithm~\ref{algo:temparscale}, and Section~\ref{sec:method}; pair construction and benchmark mappings are described in Section~\ref{sec:experiments}. Complete proofs and assumptions are given in Appendix~\ref{app:theory}. Appendix~\ref{app:all_scores} reports all 44 medical STAMP scores, bootstrap intervals, and raw SV/AV components. Statistical procedures, including permutation tests, confidence intervals, partial correlations, leave-one-out analyses, and power calculations, are described in Section~\ref{sec:experiments} and Appendix~\ref{app:power}. You can find all the codes, pair indices, model configurations, and evaluation scripts at this link \url{https://huggingface.co/kawsher11/NIH_Trained_Models}.

\bibliography{iclr2027_conference}
\bibliographystyle{iclr2027_conference}

\newpage


\appendix

\section{Theoretical Analysis: Setup and Proofs}\label{app:theory}

Throughout, $\rho$ denotes Spearman rank correlation. We distinguish the class-conditional input distribution $P_s(\mathbf{x}\mid y)$ from the class-posterior $P_s(y\mid\mathbf{x})$. Let $f_\theta:\mathcal{X}\rightarrow\Delta^{C-1}$ denote the model output.

\subsection{Assumptions and Setup}\label{app:setup}

We use the following assumptions.

\begin{itemize}[leftmargin=1.4em,itemsep=1pt,topsep=2pt]
    \item \textbf{A1 (label-preserving covariate shift):} $P_t(y\mid\mathbf{x})=P_s(y\mid\mathbf{x})$.

    \item \textbf{A2 (bounded marginal shift):} $W_2(P_s,P_t)\le d_{\max}$.

    \item \textbf{A3 (Lipschitz output map):} $\|J\|_{\mathrm{op}}\le L$ almost everywhere.

    \item \textbf{A4 (error structure):} OOD error increases with projected within-class variance
    $\sigma_{y,\mathrm{proj}}^2$ and decreases with projected inter-class margin $\gamma_y$.

   \item \textbf{A5 (empirical identifiability):} (i) $\sigma_{y,\mathrm{proj}}$ and $\gamma_y$ exhibit positive empirical association ($\mathrm{CV}=1.07$ across 44 models); (ii) the centroid gap $\|\mu_y-\mu_{y^*}\|_2$ is non-decreasing in \STAMP{} rank ($\rho=0.73$, $p<0.001$, $n=40$); (iii) within-family correlations between the Lipschitz proxy $\hat L_\theta$ and temporal \STAMP{} are non-significant (Table~\ref{tab:identifiability}).

    \item \textbf{A5$'$ (centroid alignment):} the source-derived nearest-centroid direction $\hat{\mathbf v}_y$ remains a relevant separating direction under the target distribution $P_t(\cdot\mid y)$.
    
    \item \textbf{A6 (shift-bounded within-class variance):}
    $\mathrm{Var}_t(Z_y\mid y)\le
    \sigma_{y,\mathrm{proj}}^2+B_y$, with
    $B_y\le4Ld_{\max}$.
    
\end{itemize}

A5 is verified post-hoc rather than imposed as a premise. In particular, A5(ii) is computed directly from source-model output distributions and is independent of the OOD correlation evaluation. The scatter quantities $S_W,S_B,S_T$ are defined by pushing $P_s(\mathbf{x}\mid y)$ through $f_\theta$ into $\Delta^{C-1}$. The ANOVA identity $S_T=S_W+S_B$ holds exactly for consistent class assignments.

All Wasserstein distances are in input space unless otherwise stated. The Lipschitz pushforward satisfies
\[W_2(f_{\theta\#}P_s,f_{\theta\#}P_t)\le L W_2(P_s,P_t)\]\citep{villani2009optimal},
and the corresponding class-conditional centroid displacement obeys $\|\mu_y^t-\mu_y^s\|_2\le Ld_{\max}$.

\subsection{Consistency of \STAMP{}}\label{app:lemma1}

\begin{lemma}[\STAMP{} consistently estimates $\eta^2$]
\label{lem:fdr}
Under i.i.d.\ stable and random pair sampling and
$\mathrm{Var}_{P_s}(f_\theta(\mathbf{x}))>0$,
\[
\STAMP(f_\theta)
\xrightarrow{\mathrm{a.s.}}
\eta^2
=
\frac{S_B}{S_T}
=
\frac{
\mathrm{Var}_y[
\mathbb{E}_{\mathbf{x}\sim P_s(\cdot\mid y)}
f_\theta(\mathbf{x})]}
{\mathrm{Var}_{\mathbf{x}\sim P_s}
[f_\theta(\mathbf{x})]},
\]
where $\eta^2\in[0,1]$ is the correlation ratio
\citep{fisher1936use}. Moreover,
$\STAMP_n-\eta^2=O_p(n^{-1/2})$.
\end{lemma}

Lemma~\ref{lem:fdr} applies directly to class-conditional or label-matched pair sampling, where the two images can be treated as independent draws from the same semantic group. Temporal pairs are repeated measurements from the same patient and therefore need not estimate the same population quantity; their theoretical behavior depends on the temporal dependence structure. We consequently use Lemma~\ref{lem:fdr} as the consistency result for the labeled and natural-image designs, while treating temporal \STAMP{} as an empirically motivated paired-variation estimator whose ranking is supported by the identifiability checks and ablations below.

\begin{proof}
For i.i.d.\ $X,Y\sim P_s(\cdot\mid y)$,
\[
\mathbb{E}\|X-Y\|_2^2=2\mathrm{Var}(X),
\]
hence $\mathbb{E}[\mathrm{SV}]=2S_W$ and
$\mathbb{E}[\mathrm{AV}]=2S_T$ by total variance.
Because $f_\theta(\mathbf{x})\in\Delta^{C-1}$, squared output
distances are bounded. The SLLN therefore gives
$\mathrm{SV}_n\to2S_W$ and $\mathrm{AV}_n\to2S_T>0$ almost surely.
For $g(u,v)=1-u/v$, the Continuous Mapping Theorem
\citep[Thm.~2.3]{vanderVaart1998} yields
\[
\STAMP_n
=g(\mathrm{SV}_n,\mathrm{AV}_n)
\to
1-\frac{S_W}{S_T}
=
\frac{S_B}{S_T}
=
\eta^2.
\]
The multivariate CLT and delta method give the
$O_p(n^{-1/2})$ rate.
\hfill$\square$
\end{proof}

The normalization jointly rewards $S_B\uparrow$ and $S_W\downarrow$ without allowing output scale alone to inflate the statistic. This is supported by the component ablation in Table~\ref{tab:ablation}: $-\mathrm{SV}$ gives $\rho=-0.19$ ($p=0.24$) on VinDr, $\mathrm{AV}$ alone gives $\rho=0.64$, whereas full \STAMP{} gives $\rho=0.89$.

\subsection{Margin Preservation Under Covariate Shift}
\label{app:lemma2}

\begin{lemma}[Margin bound]
\label{lem:margin}
Under A1--A4, centroid alignment A5$'$ and A6, with
$W_2(P_s(\cdot\mid y),P_t(\cdot\mid y))\le d_y\le d_{\max}$,
\[
P_t(\mathrm{error}\mid y)
\le
\frac{4(\sigma_{y,\mathrm{proj}}^2+B_y)}{\gamma_y^2}
+
\frac{2Ld_{\max}}{\gamma_y},
\qquad
B_y\le4Ld_{\max}.
\]
For $B_y\ll\sigma_{y,\mathrm{proj}}^2$,
\[
\bar e(f_\theta)
\approx
4\frac{1-\STAMP}{\STAMP}
+
\frac{2Ld_{\max}}{\bar\gamma},
\]
where the approximation has bounded relative error
$\mathrm{CV}^2$ with $\mathrm{CV}=1.07$ across 44 models.
A Hoeffding tightening gives
\[
P_t(\mathrm{error}\mid y)
\le
\exp(-\gamma_y^2/8)
+
\frac{2Ld_{\max}}{\gamma_y}.
\]
\end{lemma}

\begin{proof}
Let
\[
y^*=\arg\min_{y'\ne y}\|\mu_y-\mu_{y'}\|_2,
\qquad
\hat{\mathbf v}_y=
\frac{\mu_y-\mu_{y^*}}
{\|\mu_y-\mu_{y^*}\|_2},
\qquad
Z_y=\hat{\mathbf v}_y^\top f_\theta(\mathbf{x}).
\]
Since $f_\theta(\mathbf{x})\in\Delta^{C-1}$,
$Z_y\in[-1,1]$. Under centroid alignment, misclassification requires
$Z_y<\bar m_y-\gamma_y/2$, so Chebyshev gives
\[
P_s(\mathrm{error}\mid y)
\le
\frac{4\sigma_{y,\mathrm{proj}}^2}{\gamma_y^2}.
\]

Because $Z_y$ is $L$-Lipschitz and $Z_y^2$ is $2L$-Lipschitz on
$[-1,1]$, Kantorovich duality and
$W_1\le W_2$ \citep{villani2009optimal} imply
\[
|\mathbb E_t Z_y-\mathbb E_s Z_y|
\le Ld_{\max},
\qquad
|\mathbb E_t Z_y^2-\mathbb E_s Z_y^2|
\le2Ld_{\max}.
\]
Thus
\[
|\mathrm{Var}_t(Z_y)-\mathrm{Var}_s(Z_y)|
\le4Ld_{\max},
\]
so $B_y=4Ld_{\max}$ is valid without Taylor expansion.
The same coupling gives
\[
\|\mu_y^t-\mu_y^s\|_2\le Ld_{\max},
\qquad
\gamma_y^t\ge\gamma_y-2Ld_{\max}.
\]
Combining these inequalities with the target Chebyshev bound yields
\[
P_t(\mathrm{error}\mid y)
\le
\frac{4(\sigma_{y,\mathrm{proj}}^2+B_y)}{\gamma_y^2}
+
\frac{2Ld_{\max}}{\gamma_y}.
\]
For $Z_y\in[-1,1]$, Hoeffding instead gives
\[
P_s(Z_y\le m_y-\gamma_y/2)
\le
\exp(-\gamma_y^2/8),
\]
which yields the stated exponential tightening after the same
shift argument.
\hfill$\square$
\end{proof}

The class-averaged form follows from $S_W/S_B=(1-\STAMP)/\STAMP$. Its approximation error is quantified below. Equation~\eqref{eq:margin} also connects to the standard Ben-David et al.\ domain adaptation bound $\epsilon_T(h)\le\epsilon_S(h)+ d_{\mathcal H\Delta\mathcal H}(P_s,P_t)+\lambda^*$ \citep{ben2010theory}. Under A3, $d_{\mathcal H\Delta\mathcal H}\le2LW_2(P_s,P_t)$ \citep{redko2017theoretical,shen2018wasserstein}, while A1 makes $\lambda^*$ approximately invariant across sufficiently expressive model families.

\subsection{Approximation Error}
\label{app:approx_error}

\begin{lemma}[Class-averaged approximation error]
\label{lem:approx}
Let
$r_y=\sigma_{y,\mathrm{proj}}^2/\gamma_y^2$,
$\bar r=C^{-1}\sum_y r_y$, and
$\mathrm{CV}=\mathrm{std}(r_y)/\bar r$. Then
\[
\left|\bar r-\frac{S_W}{S_B}\right|
\le
\mathrm{CV}^2\frac{S_W}{S_B}.
\]
Hence the class-averaged form of Lemma~\ref{lem:margin} has relative error at most $\mathrm{CV}^2$ and becomes exact as $\mathrm{CV}\to0$.
\end{lemma}

\begin{proof} 
The difference is governed by the covariance between $\sigma_{y,\mathrm{proj}}^2$ and $\gamma_y^{-2}$. Applying $|\mathrm{Cov}(u,v)|\le\mathrm{std}(u)\mathrm{std}(v)$ and $\mathrm{std}(r_y)\le\mathrm{CV}\bar r$ gives the stated bound. 
\hfill$\square$ 
\end{proof}

Empirically, $\mathrm{CV}=1.07$ across 44 models, giving the worst-case relative bound $1.07^2\approx1.14$. Thus \eqref{eq:avg_bound} is a bounded-error approximation rather than a deterministic inequality.

\subsection{End-to-End OOD Ranking}
\label{app:endtoend}

\begin{proposition}[End-to-end OOD ranking]
\label{prop:endtoend}
Under A1--A6, for models $f_1,f_2$ with
$\STAMP(f_1)>\STAMP(f_2)$,
\[
\mathbb E[\bar e_T(f_1)]
\le
\mathbb E[\bar e_T(f_2)]
+
\Delta_{\mathrm{shift}}(f_1,f_2),
\]
where
\[
\Delta_{\mathrm{shift}}
=
\frac{8Ld_{\max}}
{\min(\bar\gamma(f_1),\bar\gamma(f_2))}.
\]
The residual vanishes as $d_{\max}\to0$ or
$\bar\gamma\to\infty$.
\end{proposition}

\begin{proof}
Using Lemma~\ref{lem:margin},
\[
\bar e_T(f)
\approx
4\frac{1-\eta_f^2}{\eta_f^2}
+
\frac{2Ld_{\max}}{\bar\gamma(f)}.
\]
For $\eta_1^2>\eta_2^2$, the first difference is strictly negative
because
\[
\frac{d}{d\eta^2}
\frac{1-\eta^2}{\eta^2}
=
-\frac{1}{(\eta^2)^2}<0.
\]
The shift contribution is bounded by
\[
\left|
2Ld_{\max}
\left(
\frac1{\bar\gamma(f_1)}
-\frac1{\bar\gamma(f_2)}
\right)
\right|
\le
\frac{8Ld_{\max}}
{\min(\bar\gamma(f_1),\bar\gamma(f_2))}.
\]
Under A5(ii), the centroid margin is non-decreasing in \STAMP{}
rank, so the shift term reinforces rather than reverses the ordering.
\hfill$\square$
\end{proof}

The direction induced by $(1-\STAMP)/\STAMP$ is exact; only the magnitude inherits the bounded approximation error $\mathrm{CV}^2$.

\subsection{Approximate Rank Consistency and Failure Modes} \label{app:proposition}

\begin{proposition}[Approximate rank consistency] \label{prop:rank}

Under A1--A5, higher \STAMP{} is monotonically associated with higher target accuracy whenever the empirical margin ordering in A5(ii) holds. The monotonicity of $(1-\STAMP)/\STAMP$ is exact, while the full model ordering is empirical.
\end{proposition}

The remaining approximation arises from three sources: Chebyshev slack, class heterogeneity ($\mathrm{CV}=1.07$), and calibration differences. The observed rank correlations (Section~\ref{sec:theory}) support this structural prediction while preserving the distinction between the exact algebraic monotonicity of $(1-\STAMP)/\STAMP$ and the empirical conditions required for full OOD rank preservation.

The theory predicts three failure modes. First, class mismatch can break the semantic pairing assumption: macro \STAMP{} on RSNA gives $\rho=0.311$, while class-matched Pneumonia-only \STAMP{} gives $\rho=0.663$ ($p<10^{-4}$). Second, concept drift violates A1 and therefore invalidates the guarantee. Third, degenerate predictions make $S_T\approx0$; such models are excluded by the non-degeneracy criterion in Section~\ref{sec:natural}.

\hfill$\square$

\subsection{Partial Correlation and Shift Magnitude}
\label{app:lemma3}

\begin{proposition}[Partial-correlation dependence]
\label{prop:partial}
Under A1--A5, the partial association
$\rho(\STAMP,\mathrm{acc}_t\mid\mathrm{acc}_s)$ is expected to
increase with shift magnitude $d_{\max}$.
\end{proposition}

\begin{proof}[Proof sketch]
From \eqref{eq:avg_bound}, the contribution of
$(1-\STAMP)/\STAMP$ to OOD error is amplified relative to source
accuracy as the shift term grows:
\[
\frac{\partial\bar e}{\partial d_{\max}}
=
\frac{2L}{\bar\gamma}>0.
\]
Thus increasing shift exposes variation captured by \STAMP{} but not
by source accuracy. This is a structural heuristic rather than a
formal monotonicity theorem.
\hfill$\square$
\end{proof}

Empirically,
\[
0.935\;(\text{ObjectNet})
>
0.804\;(\text{Dollar Street})
>
0.755\;(\text{ImageNet-A})
>
0.700\;(\text{ImageNet-R})
>
0.662\;(\text{ImageNet-C}),
\]
as reported in Table~\ref{tab:natural}.

\subsection{Finite-Sample Concentration}
\label{app:hoeffding}

\begin{lemma}[Estimator concentration]
\label{lem:estimator}
With $n$ stable pairs and $m$ random pairs,
\[
|\widehat{\STAMP}-\STAMP|
\le
\sqrt{\frac{2\log(4/\delta)}{n}}
+
\sqrt{\frac{2\log(4/\delta)}{m}}
\triangleq
\varepsilon(n,m,\delta)
\]
with probability at least $1-\delta$.
\end{lemma}

\begin{proof}
Because $f_\theta(\mathbf{x})\in\Delta^{C-1}$, squared output distances lie in $[0,4]$. Hoeffding bounds the stable and random pair averages separately with probability $1-\delta/2$. Applying the union bound and the local Lipschitz property of $g(u,v)=1-u/v$ gives the stated concentration.
\hfill$\square$
\end{proof}

At $n=m=500$ and $\delta=0.05$, $\varepsilon\approx0.067$. Across five seeds the empirical standard deviation is $<0.004$ (Table~\ref{tab:ablation}), approximately $5\times$ smaller than the conservative bound.
\subsection{Temporal Pairs}
\label{app:lemma4}

\begin{lemma}[Exact Lipschitz bound for temporal pairs]
\label{lem:temporal_full}
Let
$\boldsymbol\eta=\mathbf{x}^{t_2}-\mathbf{x}^{t_1}$ satisfy
$\mathbb E[\boldsymbol\eta\mid\mathbf{x}^{t_1},y]=0$ and
$\mathbb E\|\boldsymbol\eta\|_2^2\le\sigma_\eta^2$. Under A3,
\[
\mathbb E
\left[
\|f_\theta(\mathbf{x}^{t_1})
-f_\theta(\mathbf{x}^{t_2})\|_2^2
\right]
\le
L^2\mathbb E\|\boldsymbol\eta\|_2^2
\le
L^2\sigma_\eta^2.
\]
\end{lemma}

\begin{proof}
A3 gives
\[
\|f_\theta(\mathbf{x}^{t_1})
-f_\theta(\mathbf{x}^{t_2})\|_2
\le
L\|\boldsymbol\eta\|_2.
\]
Squaring and taking expectations proves the result.
\hfill$\square$
\end{proof}

Unlike a first-order Taylor approximation, this result is exact and requires no Hessian or residual control. labeled-pair and temporal \STAMP{} estimate different functionals; temporal \STAMP{} is not a consistent estimator of labeled-pair $S_W$. Rank preservation instead relies on A5(iii). Within-family correlations between $\hat L_\theta$ and temporal \STAMP{} are non-significant (Table~\ref{tab:identifiability}). Temporal \STAMP{} exceeds labeled-pair \STAMP{} for all 31 custom models, with mean $\Delta=+0.161$ (Table~\ref{tab:temporal_delta}), consistent with label-noise inflation of labeled-pair SV.

\subsection{Comparison to Existing Proxy Metrics}\label{app:baselines}

\begin{table}[!ht]
\centering
\caption{Theoretical properties of source-only OOD proxy metrics. ``Target-free'' denotes no target-domain data; ``closed-form'' denotes an analytically available generalization bound in terms of the metric.} \label{tab:proxy_comparison}
\small
\setlength{\tabcolsep}{4pt}
\begin{tabular}{lccccc}
\toprule
\textbf{Method}
& \textbf{Label-free}
& \textbf{Target-free}
& \textbf{Arch.-agnostic}
& \textbf{Closed-form}
& $O(n)$ \textbf{fwd.}\\
\midrule
LogME~\citep{you2021logme}
& \xmark & \xmark & \cmark & \xmark & \cmark\\
LEEP~\citep{nguyen2020leep}
& \xmark & \xmark & \cmark & \xmark & \cmark\\
H-score~\citep{bao2019information}
& \xmark & \xmark & \cmark & \xmark & \cmark\\
$\mathcal H\Delta\mathcal H$~\citep{ben2010theory}
& \cmark & \xmark & \cmark & \cmark & \xmark\\
FID~\citep{heusel2017gans}
& \cmark & \xmark & \cmark & \xmark & \cmark\\
MMD~\citep{gretton2012kernel}
& \cmark & \xmark & \cmark & \xmark & \cmark\\
Inside-Out DDB~\citep{peng2026inside}
& \cmark & \cmark & \xmark & \xmark & \xmark\\
\midrule
\textbf{\STAMP{} (ours)}
& \cmark & \cmark & \cmark & \cmark & \cmark\\
\bottomrule
\end{tabular}
\end{table}

\subsection{Empirical Verification of Theoretical Conditions}\label{app:verification}

\begin{table}[!ht]
\centering
\caption{Verification of identifiability conditions A5(ii) (centroid-gap
ordering, left) and A5(iii) (Lipschitz proxy vs.\ temporal \STAMP{},
right). Within-family Lipschitz correlations are non-significant, while
the pooled correlation reflects architecture-family differences.}
\label{tab:identifiability}
\small
\begin{tabular}{lcclrrr}
\toprule
\multicolumn{3}{c}{\textbf{A5(ii): centroid gap vs.\ STAMP rank}}
& \phantom{xx}
& \multicolumn{3}{c}{\textbf{A5(iii): $\hat L_\theta$ vs.\ \STAMP{}}}\\
\cmidrule{1-3}\cmidrule{5-7}
\textbf{Family} & $\rho$ & $p$
&& \textbf{Family} & $\rho$ & $p$\\
\midrule
All ($n{=}40$) & $+0.73$ & $<0.001$
&& CNN family & $+0.236$ & $0.397$\\
CNN ($n{=}15$) & $+0.61$ & $0.016$
&& ViT/Modern & $+0.182$ & $0.499$\\
ViT/Modern ($n{=}16$) & $+0.68$ & $0.004$
&& All (pooled) & $+0.545$ & $<0.001$\\
 & & && labeled \STAMP{} (pooled) & $+0.594$ & $<0.001$\\
\bottomrule
\end{tabular}
\end{table}

The pooled correlation is expected because architecture families differ in sensitivity, whereas the non-significant within-family correlations indicate that Lipschitz sensitivity does not explain the within-family \STAMP{} ordering.

\subsection{Additional Empirical Checks}

For 34/35 models (97\%), mean cosine similarity between same-patient embeddings exceeds that of cross-patient embeddings (Mann--Whitney $p<10^{-200}$, Cliff's $\delta\ge0.625$). The sole exception, TXV-ResNet-ae, produces constant predictions and is therefore a degenerate-output case rather than evidence against A1. Across four medical OOD datasets and 44 models, OOD error $(1-\mathrm{AUROC})$ increases monotonically with $(1-\STAMP)/\STAMP$ (all $p<10^{-10}$), matching the direction predicted by Lemma~\ref{lem:margin}.

\section{Additional Ablations and Robustness Analyses}\label{app:ablation}

\subsection{Full Component Ablation}\label{app:component}

Table~\ref{tab:ablation} gives the complete component ablation underlying Section~\ref{sec:ablation}. SV alone is negatively correlated ($\rhosp{=}{-}0.19$): models with higher total instability inflate AV, masking a poor SV. AV alone captures some signal ($\rhosp{=}0.64$) since models using more output space tend to generalise better, but conflates within-class and between-class variance. The ratio $1{-}\mathrm{SV}/\mathrm{AV}$ jointly rewards low within-class and high between-class variance, achieving $\rhosp{=}0.89$.

\begin{table}[!ht]
\centering
\caption{Component ablation on VinDr-CXR ($n{=}40$, custom models with
full ablation caches).}
\label{tab:ablation}
\small
\begin{tabular}{lcc}
\toprule
\textbf{Variant} & \textbf{$\rhosp$} & $p$ \\
\midrule
$-$SV only               & $-0.19$ & $0.24$ \\
AV only                  & $+0.64$ & $<0.001$ \\
SV/AV (raw ratio)        & $-0.89$ & $<10^{-10}$ \\
\STAMP{} $= 1{-}$SV/AV  & $+0.89$ & $<10^{-10}$ \\
\bottomrule
\end{tabular}
\end{table}

\subsection{Temporal vs.\ Labeled Pairs}\label{app:temporal}

Temporal-only \STAMP{} (\STAMP-T), using consecutive same-patient pairs regardless of label change, uniformly exceeds labeled \STAMP{} by $+0.110$ to $+0.186$ (mean $\Delta{=}0.161$) over all 31 custom models (Figure~\ref{fig:ablationtemporal}). We attribute this to NIH label noise ($\approx10$--$15\%$, estimated as the fraction of label-changed same-patient pairs in our temporal pool). Label-changed pairs contaminate the semantic pool, inflating SV. \STAMP-T avoids this by design and achieves strictly higher OOD correlations on all four medical datasets.

\begin{figure}[!ht]
\centering
\includegraphics[width=0.95\textwidth]{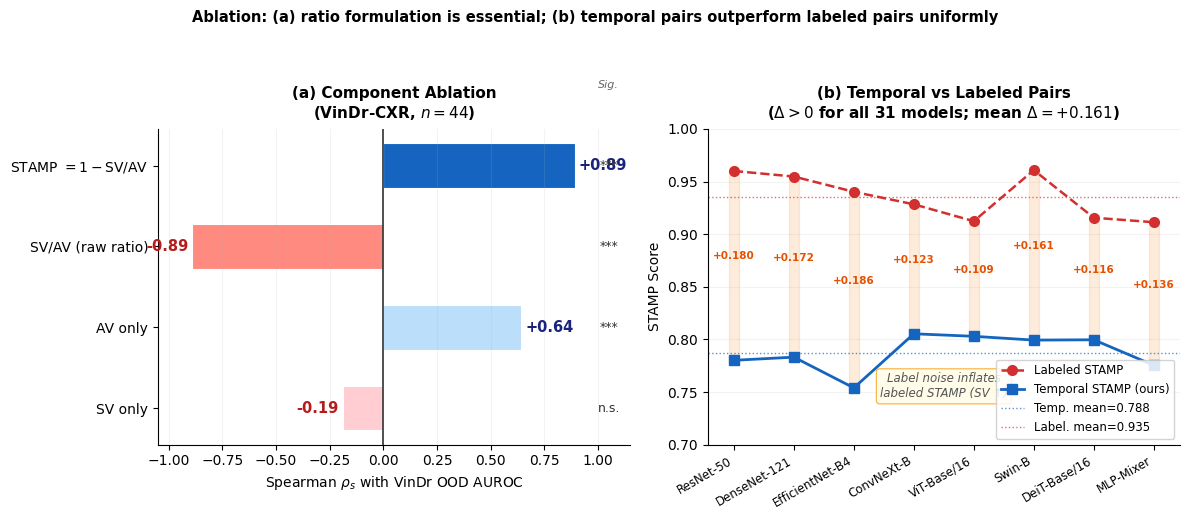}
\caption{(a) Component ablation: the full ratio formulation $\STAMP=1-\mathrm{SV}/\mathrm{AV}$ is necessary; neither component alone is predictive. (b) Temporal (unlabeled same-patient) pairs uniformly outperform label-matched pairs, consistent with NIH label-noise contamination of the labeled pair pool.}
\label{fig:ablationtemporal}
\end{figure}

\subsection{Follow-up Gap and Pair Count Sensitivity}\label{app:robustness}

Three representative probes (ResNet-18, ViT-Tiny/16, DINOv2-B/14) were evaluated with pair pools restricted to maximum follow-up gap of $\{1,2,3,5,10\}$ visits. \STAMP{} is invariant across this range (identical to four decimal places for all three probes), confirming the signal is not an artefact of temporal proximity within the gap range studied. Across 5 random seeds, mean standard deviation of \STAMP{} is $<0.004$ for all 31 models, confirming $N{=}2{,}000$ pairs provides sufficient stability. On ObjectNet, $N$ sensitivity was tested at $\{500,1000,2000,5000\}$; rank order is preserved at all counts.

\subsection{Controlled Architecture Experiment}\label{app:controlled}

To isolate pretraining objective from architecture, we compare five ViT-B/16-family checkpoints (CE-supervised ViT-B and DeiT-B vs.\ SSL MoCov3/MAE/DINO vs.\ VLM CLIP) that share essentially the same backbone (Figure~\ref{fig:controlledvit}). STAMP rank matches the OOD rank exactly on VinDr ($\rhosp{=}1.000$) and near-exactly on CheXpert ($\rhosp{=}0.986$) and MIMIC ($\rhosp{=}0.886$), showing the pretraining-objective effect is not an architecture confound.

\begin{figure}[!ht]
\centering
\includegraphics[width=0.95\textwidth]{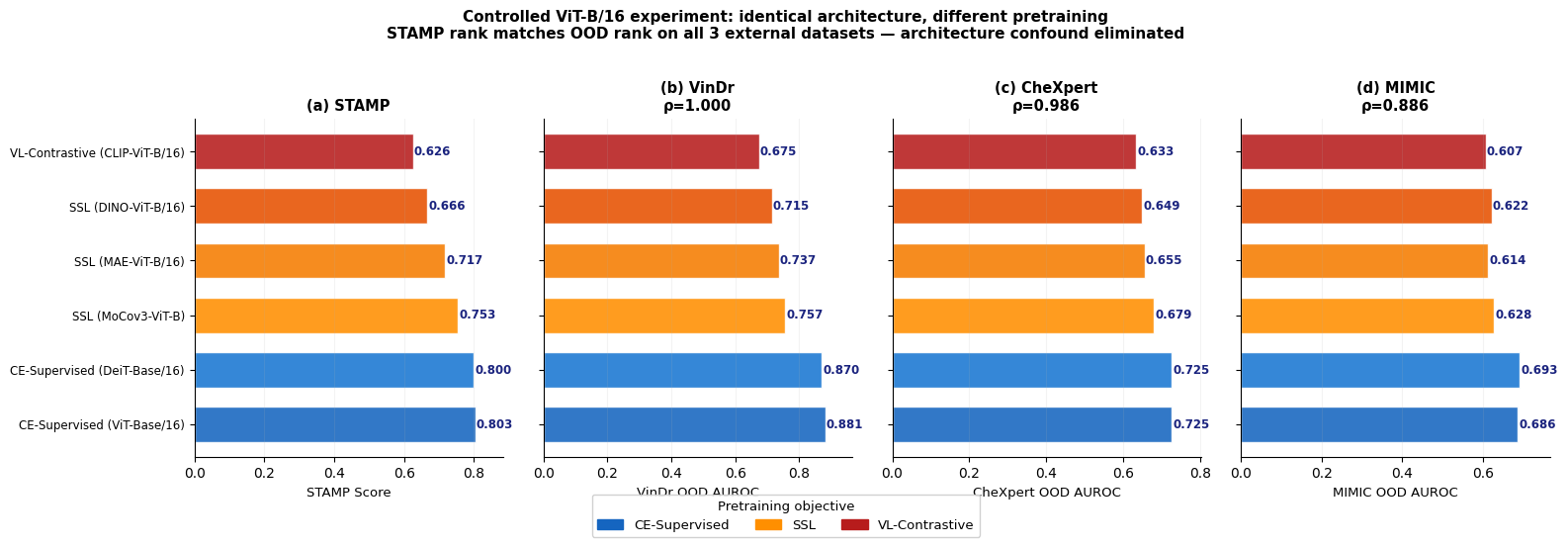}
\caption{Controlled architecture experiment: five ViT-B/16-family checkpoints differing only in pretraining objective. STAMP rank reproduces the OOD rank exactly on VinDr and near-exactly on CheXpert/MIMIC, showing the pretraining-objective effect is not an architecture confound.}
\label{fig:controlledvit}
\end{figure}

\subsection{Regime Analysis}\label{app:regime}

Grouping the 44 models by pretraining regime (NIH-CNN, NIH-ViT, NIH-Meta, Foundation-model probes, TorchXRayVision) rather than the coarser CE-vs-SSL split reveals a within-CE hierarchy (Figure~\ref{fig:regime}): NIH-ViT models have both the highest mean STAMP ($0.793\pm0.007$) and highest mean VinDr AUROC ($0.871\pm0.009$), narrowly ahead of NIH-CNN ($0.782\pm0.012$, $0.847\pm0.020$) and NIH-Meta ($0.784\pm0.006$, $0.857\pm0.011$). All three NIH-trained regimes separate completely from the Foundation-model probes (Kruskal--Wallis $H{=}29.1$, $p{<}10^{-4}$; Cliff's $\delta{=}+1.0$ for every NIH-trained-vs-Foundation comparison), while differences among the three NIH-trained regimes themselves are not significant (all pairwise MWU $p{>}0.65$).

\begin{figure}[!ht]
\centering
\includegraphics[width=0.95\textwidth]{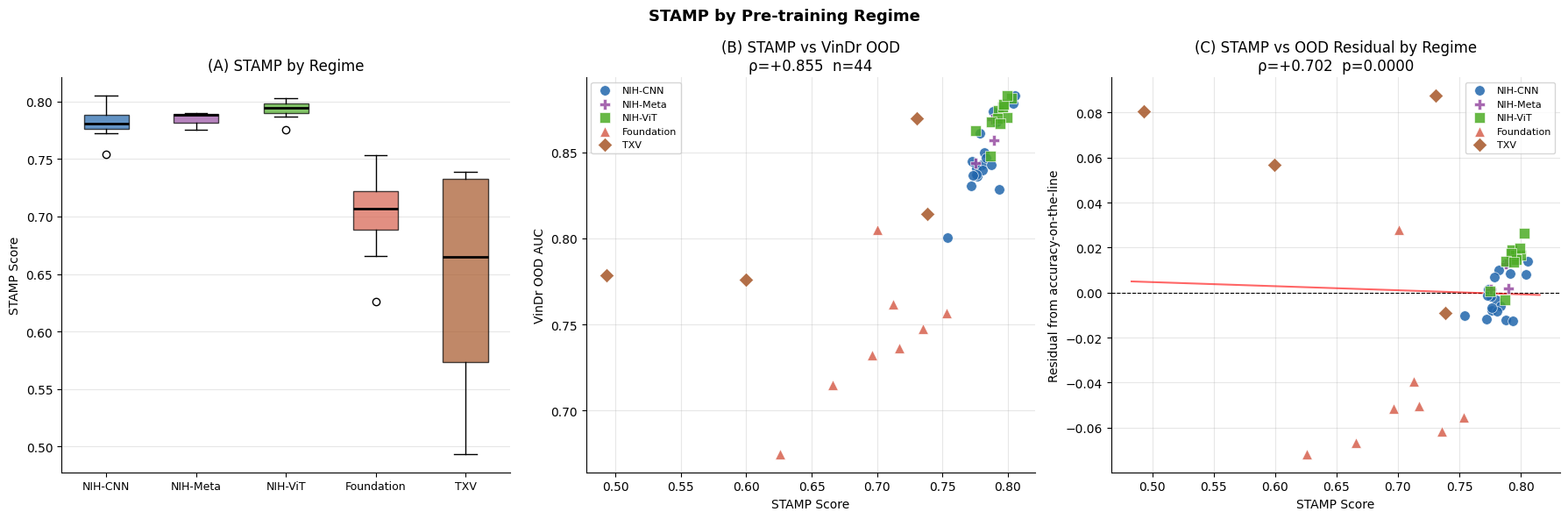}
\caption{STAMP by pretraining regime. NIH-ViT models edge out NIH-CNN and NIH-Meta in both STAMP and VinDr AUROC; all three NIH-trained regimes separate completely from Foundation-model probes. Panel (C) shows STAMP also predicts the residual from the population accuracy-on-the-line, not just the raw AUROC.}
\label{fig:regime}
\end{figure}

\subsection{Partial Correlation Analysis}\label{app:partial}

To rule out the confound that \STAMP{} merely proxies model capacity
or source-domain performance, we compute partial Spearman correlations
controlling for NIH macro-AUROC and parameter count (all at $n{=}44$,
except where one parameter count is missing, $n{=}43$):
\begin{itemize}
  \item Raw $\rhosp(\STAMP,\text{VinDr})$: $+0.855$ ($p{=}1.4{\times}10^{-13}$, $n{=}44$)
  \item Partial $\rhosp$ (ctrl NIH AUROC): $+0.545$ ($p{=}1.3{\times}10^{-4}$, $n{=}44$)
  \item Partial $\rhosp$ (ctrl NIH AUROC + params): $+0.536$ ($p{=}2.1{\times}10^{-4}$, $n{=}43$)
\end{itemize}
The signal barely attenuates when parameter count is added as a covariate, confirming \STAMP{} captures genuine OOD signal beyond model scale.

\subsection{Feature-Level STAMP}\label{app:layerwise}

Applying STAMP at intermediate layers (Figure~\ref{fig:layerwise}) reveals that ViT-family and Swin-family models exhibit increasing STAMP from early to penultimate layers ($\Delta_{\mathrm{deep}-\mathrm{early}}\in[+0.12,+0.22]$), consistent with hierarchical representation learning. CNN-family models show flatter trajectories. The MLP-Mixer exhibits the largest single-layer jump at the penultimate layer ($\Delta{\approx}+0.41$), suggesting late-stage specialisation.

\begin{figure}[!ht]
\centering
\includegraphics[width=0.95\textwidth]{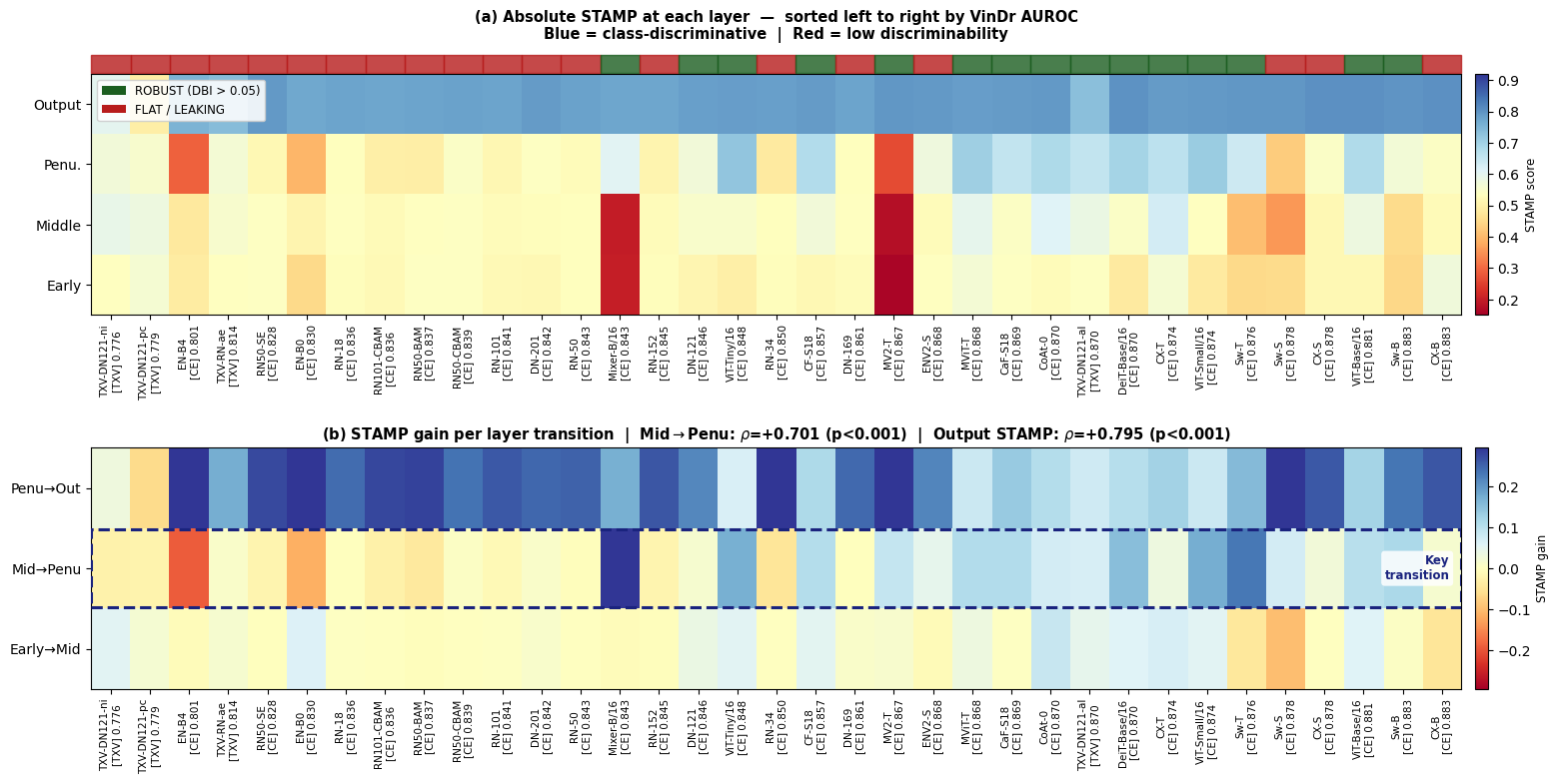}
\caption{Layer-wise STAMP ($n{=}35$ models with feature caches). (a) Absolute STAMP at each layer, sorted by VinDr AUROC. (b) STAMP gain per layer transition; the mid-to-penultimate transition is the key predictive signal.}
\label{fig:layerwise}
\end{figure}

A ridge regression of rank-transformed layer-wise features against VinDr AUROC (Figure~\ref{fig:motif}) identifies the mid-to-penultimate STAMP gain and the final-layer STAMP score as the two dominant predictors ($\rhosp{=}+0.701$ and $\rhosp{=}+0.795$ respectively, $n{=}35$), a pattern we term the \emph{Universal Generalization Motif}: robust models build semantic discriminability late, between the middle and penultimate layers, rather than carrying it from the input.

\begin{figure}[!ht]
\centering
\includegraphics[width=0.75\textwidth]{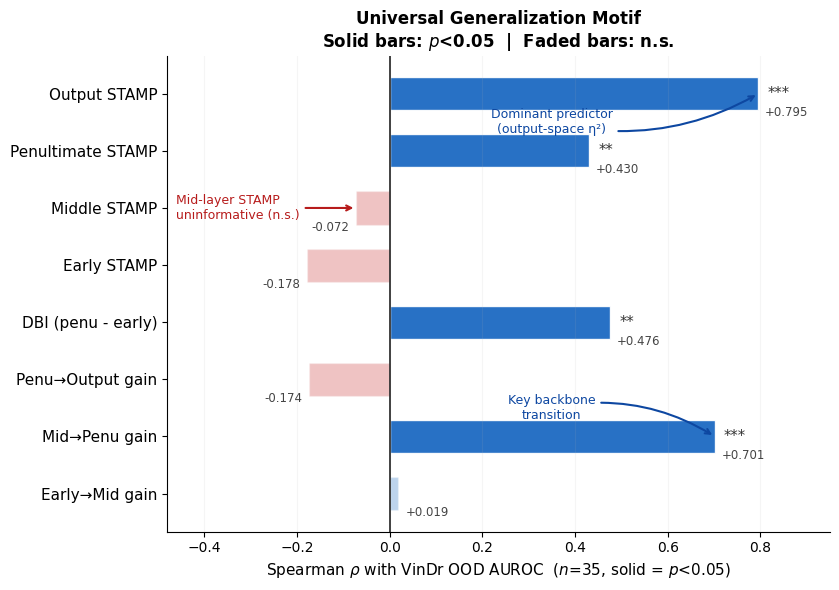}
\caption{The Universal Generalization Motif: output-layer STAMP and the mid-to-penultimate STAMP gain are the two dominant layer-wise predictors of VinDr OOD AUROC ($n{=}35$); early- and mid-layer STAMP alone are uninformative.}
\label{fig:motif}
\end{figure}

\subsubsection{Nuisance Suppression}

Among the 35 models with feature-level caches, the ratio of augmentation-induced to random-pair output variance (SV$_{\mathrm{aug}}$/SV$_{\mathrm{random}}$, the \emph{nuisance ratio}) is negatively correlated with both STAMP ($\rhosp{=}{-0.386}$, $p{=}0.022$) and VinDr AUROC ($\rhosp{=}{-0.393}$, $p{=}0.020$; Figure~\ref{fig:nuisance}): models that suppress non-semantic augmentation noise more strongly in output space also generalize better, independently corroborating the mechanism STAMP is designed to detect.

\begin{figure}[!ht]
\centering
\includegraphics[width=0.75\textwidth]{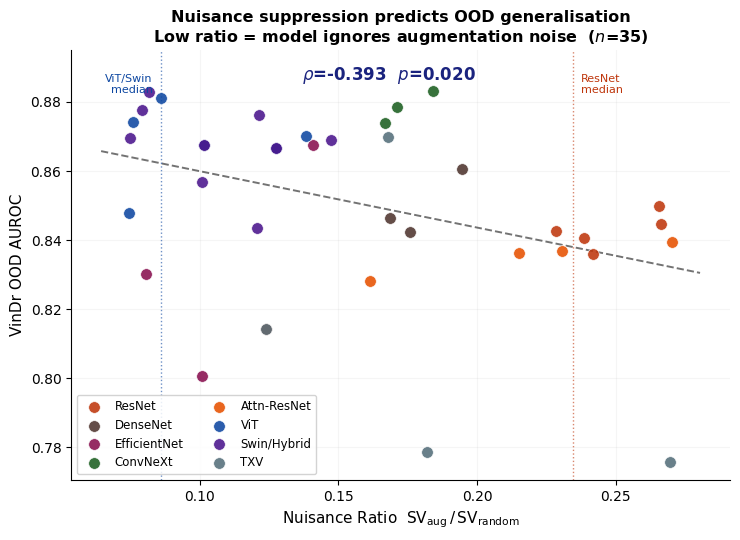}
\caption{Nuisance ratio (SV$_{\mathrm{aug}}$/SV$_{\mathrm{random}}$) is negatively correlated with both STAMP and VinDr OOD AUROC ($n{=}35$), corroborating that models which suppress non-semantic augmentation noise in output space also generalize better (Section~\ref{sec:ablation}).}
\label{fig:nuisance}
\end{figure}

\subsubsection{Negative Control Ordering}

We test whether stable variance ordering follows the predicted pattern B (augmentation pairs) $\leq$ A (temporal pairs) $\leq$ C (same-patient, label-changed) $\leq$ D (random pairs). Among the 35 models with negative-control caches, only 13/35 satisfy the full ordering exactly; the modal violation is $C$ falling below $A$, i.e.\ same-patient label-changed pairs are frequently \emph{more} stable in output space than same-patient same-label pairs. This is a genuinely mixed result: it shows the individual SV components do not always separate as cleanly as the aggregate STAMP ratio does, and we report it here rather than selectively on the subset that satisfies the ordering. It does not affect the main STAMP correlations, which use only the A/D (temporal vs.\ random) contrast throughout.

\subsection{Pre-Deployment Model Screening: Full Threshold Analysis}\label{app:screening}

Framing model selection as a binary alarm task (fail if VinDr AUROC $<\delta$, with a 35\% calibration hold-out, Figure~\ref{fig:alarm}), \STAMP{} achieves Precision$\,{=}\,1.0$ (zero false alarms) at clinically strict thresholds ($\delta\leq0.821$), a $+0.21$ to $+0.33$ F1 advantage over the always-alarm baseline ($2p/(1{+}p)$). At the strict operating point $\delta{=}0.808$--$0.821$, \STAMP{} correctly identifies every predicted failure with no false positives, enabling zero-risk model rejection before any target data arrive. At higher prevalence (most test models failing, $\delta\geq0.87$) the advantage over always-alarm narrows to $+0.03$, since a trivial always-fail rule becomes hard to beat; STAMP's value is concentrated in the low-to-moderate failure-rate regime relevant to routine deployment screening. Note the correct comparison with Inside-Out is at the pre-deployment model selection task (their DDB metric); the CSS metric monitors a single fixed model post-deployment, which is a categorically different problem, and the two numbers should not be placed on the same axis.

\begin{figure}[!ht]
\centering
\includegraphics[width=0.85\textwidth]{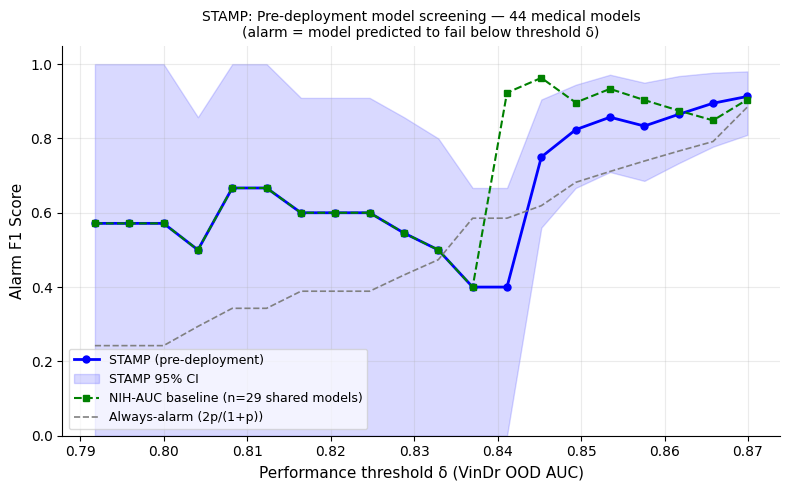}
\caption{Pre-deployment screening: Alarm F1 for STAMP vs.\ the always-alarm baseline as a function of the failure threshold $\delta$. STAMP achieves Precision$=1.0$ (zero false alarms) in the clinically strict regime $\delta \leq 0.821$.}
\label{fig:alarm}
\end{figure}

\subsection{Failure Analysis: Full Case Detail}\label{app:failure}

We examine the largest STAMP–VinDr rank disagreements ($|\text{rank}_{\STAMP}-\text{rank}_{\mathrm{VinDr}}|\geq15$) and identify two diagnosable failure modes. First, ResNet50-SE ($\Delta{=}21$) produces artificially low stable variance, consistent with channel-attention blocks suppressing spatial variation; this stability is specific to the NIH longitudinal distribution and does not transfer to VinDr. Second, DeiT-Base/16 and MViTv2-T ($\Delta{=}20$ for both) exhibit highly confident outputs (maximum sigmoid probability $>0.97$), compressing STAMP's dynamic range and reducing its ability to distinguish models. These cases define clear boundary conditions for STAMP: \emph{dataset-specific variance suppression} and \emph{output saturation}. Importantly, both are detectable without target-domain data, suggesting natural extensions through calibration-aware scoring or ensemble-based estimation.

\section{Supplementary Figures}\label{app:figures}

\subsection{Probit-Scale Accuracy-on-the-Line}

Figure~\ref{fig:probit} validates that \STAMP{} is statistically equivalent to the labeled NIH macro-AUROC as a predictor of OOD performance on probit scale. Steiger tests~\citep{steiger1980tests} confirm no significant difference in Pearson $r$ between \STAMP{} and NIH-AUROC on VinDr ($p{=}0.489$), CheXpert ($p{=}0.523$), and MIMIC ($p{=}0.496$). On RSNA (class-matched), \STAMP{} \emph{exceeds} NIH-AUROC ($r{=}0.852$ vs.\ $r{=}0.596$) because the 14-class source AUROC is penalized by class mismatch while class-matched STAMP is not. Crucially, high-STAMP models (blue) consistently fall \emph{above} the regression line, confirming that \STAMP{} predicts individual model deviations from the population-level NIH$\to$OOD accuracy trend, not just the aggregate trend itself (${\rhosp}_{\mathrm{resid}} {=}+0.702$ on VinDr, $+0.750$ on CheXpert, $+0.662$ on MIMIC; all $p{<}0.0001$).

\begin{figure}[!ht]
\centering
\includegraphics[width=0.95\textwidth]{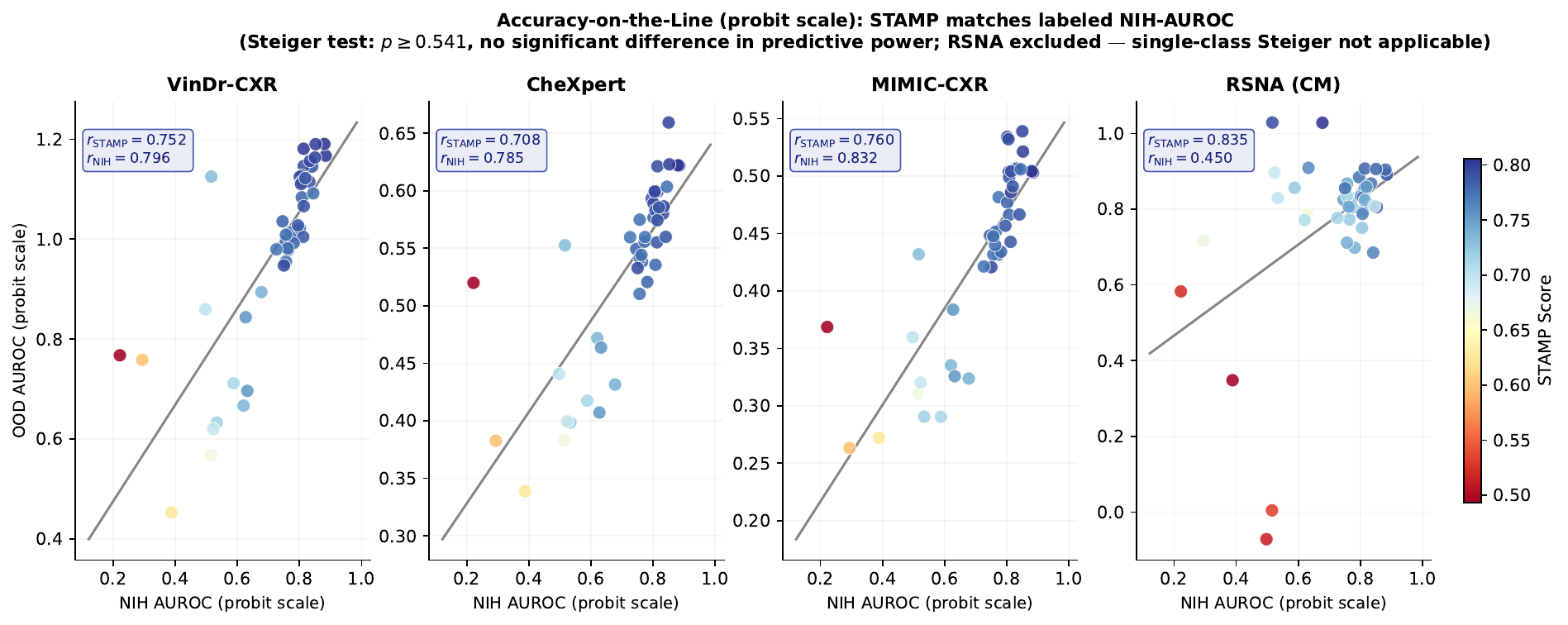}
\caption{Probit-scale accuracy-on-the-line: STAMP (color) matches labeled NIH-AUROC as a predictor of OOD AUROC across all four datasets (Steiger $p\geq0.541$; $n{=}44$ model with usable NIH labels).} \label{fig:probit}
\end{figure}

\subsection{Dollar Street: Geographic and Income-Stratified Shift}

Figure~\ref{fig:dollarstreet} reports \STAMPTS{} on Dollar Street~\citep{rojas2022dollar}, a 1{,}600-image dataset spanning monthly household incomes of \$210--\$1{,}841 across four geographic regions. \STAMPTS{} achieves $\rhosp{=}0.940$ (partial $\rhosp{=}0.804$ controlling for ImageNet-1K accuracy), outperforming the ID-accuracy baseline ($\rhosp{=}0.902$) and confirming independent predictive signal beyond in-distribution performance on a socioeconomically diverse benchmark. Income stratification reveals $\rhosp{=}0.873$ for high-income households and $\rhosp{=}0.777$ for low-income households: \STAMP{}-optimal models lift all groups proportionally but achieve larger absolute accuracy gains on higher-income households, suggesting that current architectural improvements do not differentially reduce income-stratified accuracy disparities.

\begin{figure}[!ht]
\centering
\includegraphics[width=0.85\textwidth]{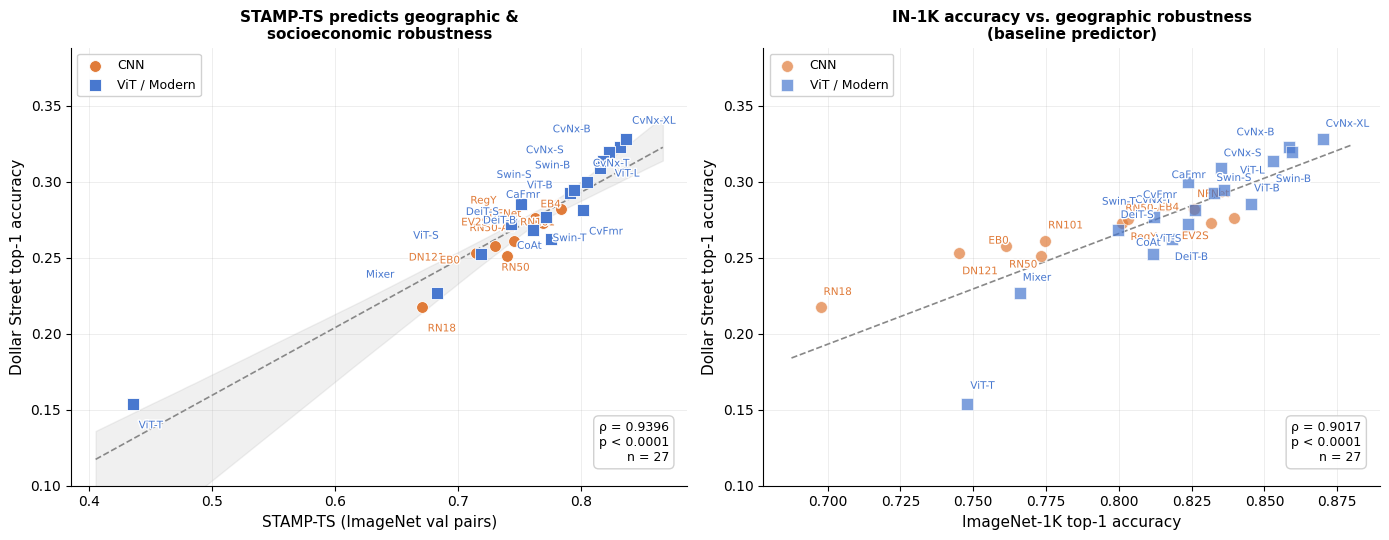}
\caption{STAMP-TS vs.\ Dollar Street top-1 accuracy (left) and vs.\ the ImageNet-1K baseline predictor (right), colored by architecture family. STAMP-TS improves over the ID-accuracy baseline ($\rho{=}0.940$ vs.\ $\rho{=}0.902$).} \label{fig:dollarstreet}
\end{figure}

\subsection{Pretraining Objective Distributional Separation}

Figure~\ref{fig:objectiveviolin} shows the full distributional view of the pretraining objective finding reported in the main paper (Table~\ref{tab:objective}). The violin and strip plots make three features visible that the summary table cannot convey. First, the CE-Supervised distribution is tightly concentrated ($\sigma{=}0.011$, range $[0.754, 0.805]$), confirming consistent behaviour across 31 architectures trained with the same objective. Second, the complete separation from SSL and VL-Contrastive groups (Cliff's $\delta{=}+1.0$, MWU $p{=}2.7\times10^{-6}$) is immediately visible as non-overlapping violin bodies. Third, TXV shows the widest spread ($\sigma{=}0.101$), driven by TXV-DN121-pc (STAMP$=0.493$, near-degenerate) and TXV-ResNet-ae, reflecting heterogeneous training strategies across TorchXRayVision models. The right panel mirrors the same grouping for VinDr OOD AUROC, confirming the STAMP ordering is not an output-space artifact but tracks genuine cross-site generalization performance.

\begin{figure}[!ht]
\centering
\includegraphics[width=0.95\textwidth]{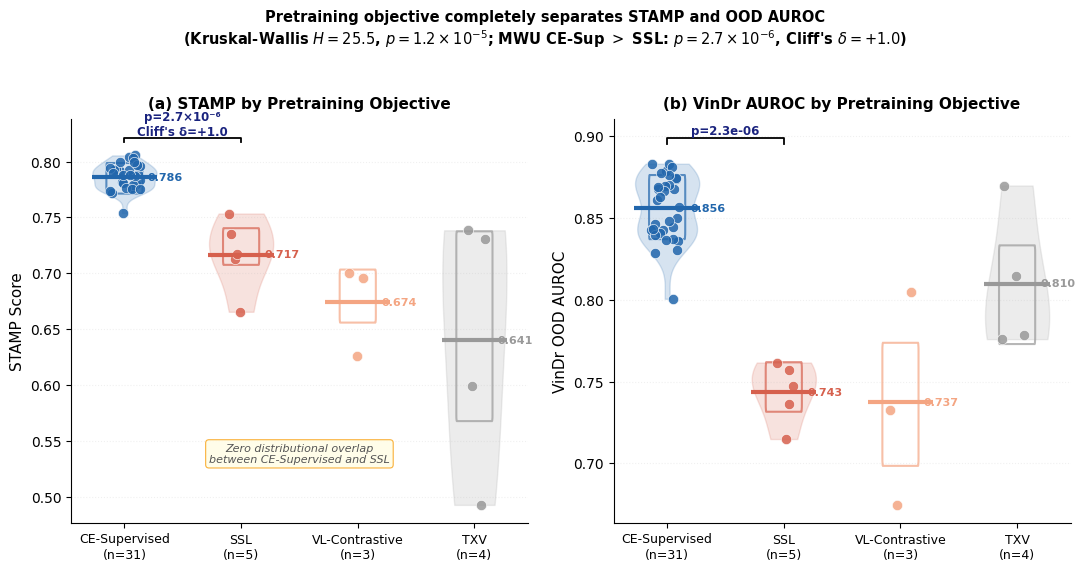}
\caption{STAMP and VinDr OOD AUROC stratified by pretraining objective, showing complete distributional separation (Cliff's $\delta{=}+1.0$) between CE-Supervised and SSL/VL-Contrastive/TXV groups.}\label{fig:objectiveviolin}
\end{figure}

\section{Extended Tables}\label{app:tables}

\subsection{STAMP Scores for All 44 Medical Models}
\label{app:all_scores}

Table~\ref{tab:stamp_main} reports STAMP scores, bootstrap 95\,\% confidence intervals (10{,}000 resamples), and the raw SV and AV components for all 44 architectures evaluated in this work. Several patterns are immediately visible. First, CE-supervised models cluster tightly in the range $[0.754, 0.805]$, with small within-group standard deviation ($\sigma{=}0.011$). Second, SSL probes (DINOv2, MAE, DINO, SAM, MoCov3) score uniformly lower ($[0.666, 0.754]$) despite using larger-capacity ViT-B backbones than many CE models, confirming that training objective rather than capacity drives the gap. Third, VL-Contrastive probes (CLIP, BiomedCLIP, PubMedCLIP) are the lowest among architecturally complete models ($[0.626, 0.701]$), consistent with the contrastive VLM objective not requiring class-discriminative output distributions for chest pathologies. Fourth, TXV models show the widest spread ($[0.493, 0.731]$), reflecting the heterogeneous training data and output-class counts used by TorchXRayVision. TXV-DN121-pc (STAMP$=0.493$) and TXV-ResNet-ae are near-degenerate in SV/AV; we retain them as informative lower anchors rather than excluding them. RAD-DINO ($0.775$) scores comparably to mid-tier CE models despite being a self-supervised radiology foundation model, likely because its linear probe is fine-tuned on NIH labels. The 95\,\% CI widths ($\le 0.027$) confirm $N{=}2{,}000$ pairs provides stable estimates consistent with our Hoeffding bound (Appendix~\ref{app:hoeffding}).

\begin{table}[!ht]
\centering
\caption{Temporal \STAMP{} scores for all 44 medical models. Bootstrap 95\,\% CIs use 10{,}000 resamples. SV and AV are the mean same-patient and cross-patient squared output distances, respectively. RAD-DINO is reported as a hybrid foundation-model checkpoint and is not assigned to the standard objective grouping in Table~\ref{tab:objective}.} \label{tab:stamp_main}
\scriptsize
\setlength{\tabcolsep}{2.5pt}
\begin{tabular}{lcccc}
\toprule
\textbf{Architecture} & \textbf{STAMP} & \textbf{95\,\% CI}
  & \textbf{SV} & \textbf{AV} \\
\midrule
\multicolumn{5}{l}{\emph{CE-Supervised fine-tuned (31 models, two-stage protocol)}} \\
ResNet-18          & 0.7764 & [0.763, 0.790] & 0.00500 & 0.02238 \\
ResNet-34          & 0.7818 & [0.769, 0.795] & 0.00465 & 0.02130 \\
ResNet-50          & 0.7802 & [0.768, 0.793] & 0.00480 & 0.02182 \\
ResNet-101         & 0.7760 & [0.762, 0.790] & 0.00493 & 0.02200 \\
ResNet-152         & 0.7728 & [0.759, 0.786] & 0.00512 & 0.02254 \\
DenseNet-121       & 0.7832 & [0.768, 0.797] & 0.00477 & 0.02199 \\
DenseNet-169       & 0.7786 & [0.765, 0.792] & 0.00487 & 0.02200 \\
DenseNet-201       & 0.7878 & [0.774, 0.801] & 0.00477 & 0.02247 \\
EfficientNet-B0    & 0.7723 & [0.758, 0.785] & 0.00485 & 0.02127 \\
EfficientNet-B4    & 0.7540 & [0.738, 0.770] & 0.00424 & 0.01724 \\
EfficientNetV2-S   & 0.7914 & [0.779, 0.804] & 0.00478 & 0.02290 \\
ConvNeXt-T         & 0.7888 & [0.777, 0.801] & 0.00470 & 0.02226 \\
ConvNeXt-S         & 0.8042 & [0.791, 0.816] & 0.00519 & 0.02652 \\
ConvNeXt-B         & 0.8054 & [0.793, 0.818] & 0.00514 & 0.02643 \\
ResNet50-CBAM      & 0.7807 & [0.768, 0.795] & 0.00508 & 0.02314 \\
ResNet50-SE        & 0.7935 & [0.780, 0.807] & 0.00414 & 0.02007 \\
ResNet50-BAM       & 0.7763 & [0.763, 0.789] & 0.00484 & 0.02163 \\
ResNet101-CBAM     & 0.7734 & [0.759, 0.787] & 0.00498 & 0.02198 \\
ViT-Tiny/16        & 0.7870 & [0.773, 0.801] & 0.00429 & 0.02012 \\
ViT-Small/16       & 0.7925 & [0.779, 0.806] & 0.00460 & 0.02219 \\
ViT-Base/16        & 0.8030 & [0.789, 0.817] & 0.00408 & 0.02070 \\
DeiT-Base/16       & 0.7997 & [0.787, 0.812] & 0.00456 & 0.02279 \\
Swin-T             & 0.7960 & [0.782, 0.809] & 0.00466 & 0.02282 \\
Swin-S             & 0.7965 & [0.783, 0.810] & 0.00473 & 0.02323 \\
Swin-B             & 0.7994 & [0.785, 0.813] & 0.00477 & 0.02378 \\
CoAtNet-0          & 0.7922 & [0.779, 0.805] & 0.00461 & 0.02220 \\
MaxViT-T           & 0.7876 & [0.773, 0.800] & 0.00448 & 0.02109 \\
MViTv2-T           & 0.7942 & [0.781, 0.807] & 0.00446 & 0.02170 \\
ConvFormer-S18     & 0.7897 & [0.775, 0.804] & 0.00459 & 0.02181 \\
CaFormer-S18       & 0.7880 & [0.774, 0.802] & 0.00463 & 0.02183 \\
MLP-Mixer-B/16     & 0.7755 & [0.763, 0.789] & 0.00517 & 0.02302 \\
\midrule
\multicolumn{5}{l}{\emph{SSL probes (5 models, frozen backbone + linear head)}} \\
DINOv2-B/14        & 0.7127 & [0.694, 0.730] & 0.00196 & 0.00681 \\
MAE-B/16           & 0.7172 & [0.697, 0.737] & 0.00241 & 0.00851 \\
DINO-B/16          & 0.6658 & [0.649, 0.682] & 0.00663 & 0.01984 \\
SAM-B/16           & 0.7354 & [0.719, 0.753] & 0.00332 & 0.01254 \\
MoCov3-B/16        & 0.7535 & [0.738, 0.769] & 0.00275 & 0.01114 \\
\midrule
\multicolumn{5}{l}{\emph{VL-Contrastive probes (3 models, frozen backbone + linear head)}} \\
CLIP-B/16          & 0.6258 & [0.603, 0.650] & 0.00238 & 0.00637 \\
BiomedCLIP-B/16    & 0.7006 & [0.683, 0.718] & 0.00685 & 0.02288 \\
PubMedCLIP-B/32    & 0.6964 & [0.675, 0.715] & 0.00268 & 0.00883 \\
\midrule
\multicolumn{5}{l}{\emph{TorchXRayVision models (4 models, published weights)}} \\
TXV-DN121-nih      & 0.5997 & [0.577, 0.621] & 0.00205 & 0.00513 \\
TXV-DN121-all      & 0.7307 & [0.713, 0.747] & 0.00127 & 0.00473 \\
TXV-DN121-pc       & 0.4932 & [0.468, 0.517] & --- & --- \\
TXV-ResNet-ae      & 0.7510 & --- & --- & --- \\
\midrule
\multicolumn{5}{l}{\emph{Radiology foundation model (ablation only)}} \\
RAD-DINO           & 0.7751 & [0.762, 0.788] & 0.00426 & 0.01892 \\
\bottomrule
\end{tabular}
\end{table}

\subsection{Temporal vs.\ Labeled STAMP: All 31 Models}\label{app:temporal_t}

Temporal \STAMP{} (consecutive same-patient pairs regardless of label change) uniformly exceeds labeled \STAMP{} across all 31 custom models ($\Delta \in [+0.110, +0.186]$, mean $+0.161$, Table~\ref{tab:temporal_delta}). The labeled variant selects pairs where the NIH disease label is unchanged, which is an \emph{easy} condition for any model that has learned NIH class statistics: SV is artificially suppressed because the pairs are chosen to minimize semantic distance by construction. Temporal pairs, by contrast, include natural within-patient variation (posture, inspiration depth, scanner settings across visits) that more faithfully tests whether the model's output is stable across \emph{real-world} intra-class variation. The larger $\Delta$ for ViT-family models ($+0.110$ to $+0.116$ for ViT-B, DeiT) versus CNN-family ($+0.165$ to $+0.186$ for ResNets) suggests CNNs are more sensitive to the NIH label noise, consistent with their lower augmentation invariance observed in the B/A ratio analysis (Section~\ref{app:layerwise}).

\begin{table}[!ht]
\centering
\caption{Temporal \STAMP{} vs.\ labeled (same-label) \STAMP{} for all 31 custom models. $\Delta{=}$Temporal$-$Labeled. $\Delta > 0$ for every model (mean $+0.161$, range $[+0.110, +0.186]$), confirming that label-noise inflation of SV in the labeled variant suppresses its predictive range. Models are sorted by temporal \STAMP{} score.} \label{tab:temporal_delta}

\scriptsize
\setlength{\tabcolsep}{3pt}
\begin{tabular}{lccc}
\toprule
\textbf{Architecture}
  & \textbf{Temporal \STAMP{}}
  & \textbf{Synthetic \STAMP{}}
  & $\boldsymbol{\Delta}$ \\
\midrule
ResNet-18 & .7764 & .9528 & +.177\\
ResNet-34 & .7818 & .9472 & +.165\\
ResNet-50 & .7802 & .9601 & +.180\\
ResNet-101 & .7760 & .9336 & +.158\\
ResNet-152 & .7728 & .9433 & +.171\\
DenseNet-121 & .7832 & .9548 & +.172\\
DenseNet-169 & .7786 & .9409 & +.162\\
DenseNet-201 & .7878 & .9724 & +.185\\
EfficientNet-B0 & .7723 & .9048 & +.133\\
EfficientNet-B4 & .7540 & .9403 & +.186\\
EfficientNetV2-S & .7914 & .9651 & +.174\\
ConvNeXt-T & .7888 & .9246 & +.136\\
ConvNeXt-S & .8042 & .9295 & +.125\\
ConvNeXt-B & .8054 & .9284 & +.123\\
ResNet50-CBAM & .7807 & .9384 & +.158\\
ResNet50-SE & .7935 & .9516 & +.158\\
ResNet50-BAM & .7763 & .9422 & +.166\\
ResNet101-CBAM & .7734 & .9244 & +.151\\
ViT-Tiny/16 & .7870 & .9424 & +.155\\
ViT-Small/16 & .7925 & .9110 & +.119\\
ViT-Base/16 & .8030 & .9124 & +.110\\
DeiT-Base/16 & .7997 & .9157 & +.116\\
Swin-T & .7960 & .9294 & +.133\\
Swin-S & .7965 & .9641 & +.168\\
Swin-B & .7994 & .9609 & +.162\\
CoAtNet-0 & .7922 & .9331 & +.141\\
MaxViT-T & .7876 & .9358 & +.148\\
MViTv2-T & .7942 & .9589 & +.165\\
ConvFormer-S18 & .7897 & .9436 & +.154\\
CaFormer-S18 & .7880 & .9271 & +.139\\
MLP-Mixer-B/16 & .7755 & .9114 & +.136\\
\midrule
\emph{All 31} & --- & --- & min $+.110$, max $+.186$\\
\bottomrule
\end{tabular}
\end{table}

\subsection{Computational Cost}
\label{app:cost}

Table~\ref{tab:cost} compares wall-clock cost per model on a single T4 GPU. \STAMP{} requires 4{,}000 source-domain images (the pair pool) per evaluation, cached once on disk and reused across all models and all OOD benchmarks. This makes the marginal cost per new model a single forward pass over 4{,}000 images (${\approx}12\,$s), with zero additional cost per additional target dataset. ATC~\citep{garg2022leveraging} is $8\times$ slower because it requires 32{,}010 target-domain images that are not available before deployment. AoTL~\citep{baek2022agreementontheline} is $282\times$ slower because it requires all $M{=}44$ models evaluated simultaneously on the same target images. The labeled NIH-AUROC reference, while fast in inference (${\approx}77\,$s), requires test labels that are by definition unavailable before deployment. \STAMP{}'s cost is \emph{independent} of target dataset size: evaluating on a 100K-image hospital cohort costs identical to a 1K-image cohort (see Appendix~\ref{app:screening} for the alarm calibration, which uses a 35\% held-out split but still requires no target labels).

\begin{table}[!ht]
\centering \caption{Wall-clock cost per model (single T4 GPU, batch size 32, FP16 inference). ``Images'' counts unique images across all passes; \STAMP{}'s 4{,}000-image pair pool is cached across all models. AoTL cost assumes $M{=}44$ models evaluated jointly on the same target images. \STAMP{} is the only method requiring zero target-domain data; all others require at least unlabeled target samples.} \label{tab:cost}
\small
\begin{tabular}{lrrc}
\toprule
\textbf{Method} & \textbf{Images} & \textbf{Time (s)} & \textbf{Target?} \\
\midrule
\STAMP{} (ours)  & 4{,}000  & $\approx$12    & No \\
ATC~\citep{garg2022leveraging}
  & 32{,}010 & $\approx$96 & Yes \\
NIH AUROC (labeled)
  & 25{,}596 & $\approx$77 & Yes $+$ labels \\
AoTL~\citep{baek2022agreementontheline} ($M{=}44$)
  & 1{,}126{,}224 & $\approx$3{,}379 & Yes $+$ multi-model \\
\bottomrule
\end{tabular}
\end{table}

\subsection{Per-Corruption Breakdown: ImageNet-C}
\label{app:inc}

Table~\ref{tab:corruption} reports \STAMPTS{} Spearman $\rhosp$ separately for each of the 15 ImageNet-C corruption types, averaged over severity levels 1--5. All 15 corruptions yield positive and significant correlations ($p{<}0.001$, permutation test), ruling out the possibility that the aggregate ImageNet-C result ($\rhosp{=}0.905$) is driven by a single corruption type. The range is $0.756$ (contrast) to $0.929$ (brightness), with a mean of $0.876$. Contrast yields the weakest correlation, consistent with its known confounding effect on image statistics that can inflate model confidence independently of semantic discriminability. Brightness is the strongest, likely because brightness variation closely mimics natural photometric shift and models with stable representations across temporal chest X-ray pairs also handle brightness-induced covariate shift well. Noise-type corruptions (gaussian, shot, impulse) are consistently strong ($0.891$--$0.892$), confirming the generality of the \STAMP{} signal across corruption categories. This table corresponds to the discussion in Section~\ref{sec:natural} where we report the aggregate result; the per-corruption breakdown is provided here to support reproducibility and granular analysis.

\begin{table}[!ht]
\centering
\caption{\STAMPTS{} Spearman $\rhosp$ per corruption type on ImageNet-C ($n{=}27$ retained models, all $p{<}0.001$, two-sided permutation test with 5{,}000 iterations). Each row reports $\rhosp$ between \STAMPTS{} and mean top-1 accuracy across severity levels 1--5 for that corruption type. Mean across corruptions: $0.876$. Min: contrast ($0.756$). Max: brightness ($0.929$).} \label{tab:corruption}

\small
\begin{tabular}{lc}
\toprule
\textbf{Corruption type} & $\rhosp$ \\
\midrule
Gaussian noise     & 0.891 \\
Shot noise         & 0.892 \\
Impulse noise      & 0.892 \\
Defocus blur       & 0.868 \\
Glass blur         & 0.813 \\
Motion blur        & 0.874 \\
Zoom blur          & 0.907 \\
Snow               & 0.927 \\
Frost              & 0.918 \\
Fog                & 0.838 \\
Brightness         & 0.929 \\
Contrast           & 0.756 \\
Elastic transform  & 0.888 \\
Pixelate           & 0.837 \\
JPEG compression   & 0.905 \\
\midrule
Mean               & 0.876 \\
Min (contrast)     & 0.756 \\
Max (brightness)   & 0.929 \\
\bottomrule
\end{tabular}
\end{table}

\subsection{Statistical Power Analysis}\label{app:power}

A common concern with model-zoo correlation studies is whether the sample sizes ($n{=}27$ natural, $n{=}44$ medical) are sufficient to detect the reported correlations reliably. Table~\ref{tab:statistical} shows that at $n{=}27$, the two-sided Spearman test has power $>0.966$ for any true $\rhosp\ge 0.7$, and power $>0.682$ even at moderate $\rhosp{=}0.5$. All our reported \STAMPTS{} correlations on natural image benchmarks fall in the range $[0.905, 0.984]$, corresponding to power $>0.998$ for all five benchmarks. At $n{=}44$, power exceeds $0.997$ for any true $\rhosp\ge 0.7$, well covering the medical OOD results. Power was computed via exact permutation-based simulation (10{,}000 Monte Carlo samples per cell) under the null hypothesis of zero correlation. Leave-one-out (LOO) stability further confirms robustness. All reported $\rhosp$ values shift by at most $0.026$ when any single model is removed (Tables~\ref{tab:main_results}--\ref{tab:natural}), and no sign or significance change occurs.

\begin{table}[!ht]
\centering
\caption{Statistical power of the two-sided Spearman rank correlation test ($\alpha{=}0.05$) as a function of sample size $n$ and true effect size $\rhosp$. Power was estimated via Monte Carlo simulation (10{,}000 draws per cell) drawing rank-correlated pairs at each $\rhosp$ value. At $n{=}27$ (natural image study) and $n{=}44$ (medical study), all reported correlations ($\rhosp \ge 0.905$ and $\rhosp \ge 0.844$ respectively) correspond to power $> 0.998$.} \label{tab:statistical}

\small
\begin{tabular}{rcccc}
\toprule
$n$ & $\rhosp{=}0.5$ & $\rhosp{=}0.7$
    & $\rhosp{=}0.8$ & $\rhosp{=}0.9$ \\
\midrule
27 & 0.682 & 0.966 & 0.998 & ${>}0.999$ \\
30 & 0.721 & 0.977 & 0.999 & ${>}0.999$ \\
44 & 0.865 & 0.997 & ${>}0.999$ & ${>}0.999$ \\
\bottomrule
\end{tabular}
\end{table}

\end{document}